\documentclass{article} 
\usepackage[preprint]{colm2025_conference}

\usepackage{lineno}

\usepackage{microtype}
\usepackage{hyperref}
\usepackage{url}
\usepackage{booktabs}
\usepackage{wrapfig}

\usepackage{microtype}
\usepackage{graphicx}
\usepackage{subfigure}
\usepackage{booktabs} 

\usepackage{arydshln}

\usepackage{amsmath, amsthm, amssymb}
\usepackage{longtable}
\usepackage{amsmath}
\usepackage{graphicx, subfigure}
\usepackage{booktabs}

\usepackage{algorithm}
\usepackage{algpseudocode}
\DeclareMathOperator*{\argmax}{arg\,max}
\DeclareMathOperator*{\argmin}{arg\,min}
\newtheorem{prop}{Proposition}
\usepackage{adjustbox} 

\usepackage{amsmath,amssymb}
\usepackage{amsthm}
\newtheorem{example}{Example}

\definecolor{darkblue}{rgb}{0, 0, 0.5}
\hypersetup{colorlinks=true, citecolor=darkblue, linkcolor=darkblue, urlcolor=darkblue}

\title{Learning Explainable Dense Reward Shapes via Bayesian Optimization}

\author{Ryan Koo\textsuperscript{1}, Ian Yang\textsuperscript{2}, Vipul Raheja\textsuperscript{3}, Mingyi Hong\textsuperscript{1}, Kwang-Sung Jun\textsuperscript{4}, Dongyeop Kang\textsuperscript{1}  \\
University of Minnesota\textsuperscript{1}, Georgia Institute of Technology\textsuperscript{2}, Grammarly\textsuperscript{3}, University of Arizona\textsuperscript{4} \\
\texttt{\{koo00017, mhong, dongyeop\}@umn.edu}, \texttt{iyang30@gatech.edu}, \texttt{vipul.raheja@grammarly.com}, \\\texttt{kjun@cs.arizona.edu}
}

\begin{document}

\ifcolmsubmission
\linenumbers
\fi

\maketitle


\begin{abstract}
Current reinforcement learning from human feedback (RLHF) pipelines for large language model (LLM) alignment typically assign scalar rewards to sequences, using the final token as a surrogate indicator for the quality of the entire sequence. However, this leads to sparse feedback and suboptimal token-level credit assignment.
In this work, we frame reward shaping as an optimization problem focused on token-level credit assignment. We propose a reward-shaping function leveraging explainability methods such as SHAP and LIME to estimate per-token rewards from the reward model. To learn parameters of this shaping function, we employ a bilevel optimization framework that integrates Bayesian Optimization and policy training to handle noise from the token reward estimates.
Our experiments show that achieving a better balance of token-level reward attribution leads to performance improvements over baselines on downstream tasks and finds an optimal policy faster during training. Furthermore, we show theoretically that explainability methods that are feature additive attribution functions maintain the optimal policy as the original reward.
The code is publicly available.\footnote{\url{https://github.com/minnesotanlp/explainable-dense-rewards}}
\end{abstract}

\section{Introduction}
One of the fundamental challenges in reinforcement learning (RL) arises from the sparsity of rewards, where feedback signals are typically only provided at the end of a trajectory, with little to no evaluative information about the intermediate states. 
As a result, this limitation leads to gaps about the favorability of certain intermediate states, as the agent lacks the required granular feedback regarding which actions were beneficial to the outcome. A similar challenge has been observed in many recent applications of reinforcement learning from human feedback (RLHF), where sparse rewards are common \citep{zheng2023secretsrlhflargelanguage, chaudhari2024rlhfdecipheredcriticalanalysis}.  

Given the sequential nature of language models and the token-level value function typically optimized in RLHF \citep{zhong2025dpomeetspporeinforced, rafailov2024rqlanguagemodel}, it is crucial to determine the contribution of individual tokens to the overall reward. The challenge lies in the fact that the reward is often awarded at the end of sequence generation and represents the quality of the entire sequence as a scalar, which is known to be unstable \citep{razin2024vanishinggradientsreinforcementfinetuning, engstrom2020implementationmattersdeeppolicy} and encodes a low-bandwidth signal that does not help determine the relative quality of intermediate tokens. Hence, to address this limitation, it can be beneficial to assign token-level rewards for fine-grained feedback to the policy \citep{wu2023finegrainedhumanfeedbackgives, xie2024text2rewardrewardshapinglanguage}.

However, assigning scores to tokens as they get autoregressively generated is computationally intensive, and collecting a large set of fine-grained human annotations for a supervised learning setup is expensive and subject to high disagreement. Thus, some works have explored "reward shaping" techniques to transform sparse rewards into \textit{dense rewards} \citep{Sutton1998, ng1999-policy-invariance}, thus facilitating more efficient and interpretable optimization, as well as providing finer-grained control over intermediate decisions. For instance, recent attempts examine \textit{process rewards} \citep{lightman2023letsverifystepstep, uesato2022solvingmathwordproblems}, which provide intermediate feedback for chain-of-thought generations or directly utilize the attention map \citep{chan2024denserewardfreereinforcement} of the reward model to redistribute the reward. However, these steps are complex and often require high-quality human feedback, and the attention on each token may not be directly correlated to its output as an explanation \citep{jain2019attentionexplanation}. 

In this work, we propose a two-part approach to densify sparse rewards: (1) using explainability methods to construct a dense reward signal, and (2) using Bayesian optimization to learn the weights for a new reward shaping function composed of the explainability scores. We frame reward shaping as a bi-level optimization problem: at the higher level, we optimize the coefficients of the shaped reward function, while at the lower level, we learn the corresponding optimal policy, as illustrated in Figure \ref{fig:method}. First, we estimate token-level contributions using explainability techniques such as LIME \citep{ribeiro2016whyitrustyou} or SHAPley values \citep{lundberg2017shapley}. Next, since these explainability methods are known to be sensitive to noise \citep{li2020-shap-uncertainty}, we treat them as uncertain estimates and seek to learn an optimal weighting scheme over them. However, due to the complexity of the reward landscape, exhaustively evaluating all possible weight configurations is computationally infeasible. To address this, we treat the problem as black-box optimization and employ Bayesian optimization to learn the best reward coefficients as a natural method robust to noisy objective functions \citep{fröhlich2020noisyinputentropysearchefficient, daulton2022robustmultiobjectivebayesianoptimization}.

We demonstrate that explainability offers a natural way to extract more information from the reward model to densify rewards and satisfies as \textit{potential-based} reward shaping \citep{ng1999-policy-invariance} so that we satisfy \textit{policy invariance} or that the optimal policy with the original reward function remains unchanged. Furthermore, we introduce a new optimization setup incorporating different sources of token-level information and show that Bayesian Optimization can aid in learning the best reward-shaping function. Empirically, we show that explainability methods positively impact the RL training compared to sparse rewards through accelerating learning and more stable updates in the value function. We also show that adding Bayesian Optimization to properly shape rewards improves generation quality on downstream tasks compared to its naive setups. 

\begin{figure*}[t]
    \centering
    \includegraphics[width=\linewidth
        , trim={1.0cm 12cm 1cm 8.5cm}, clip
        ]{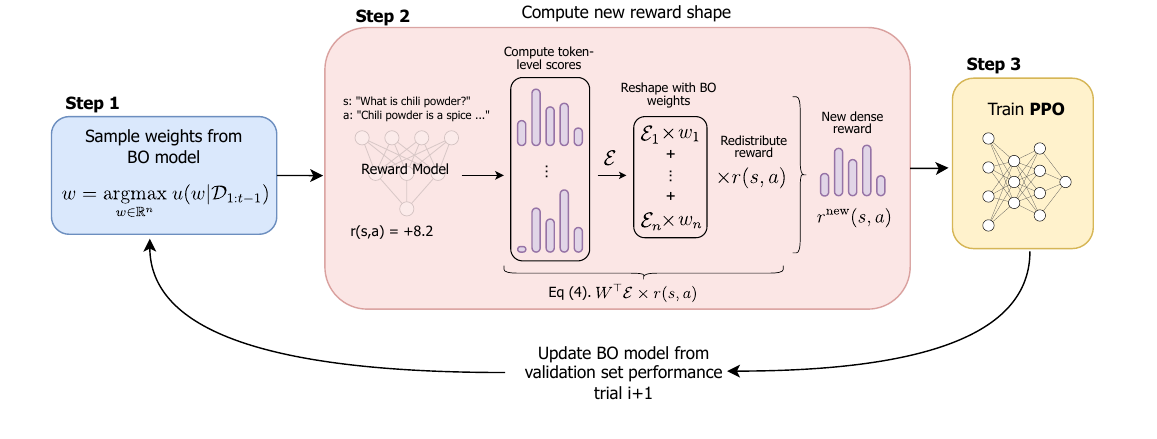}
    \vspace{-0.75cm}
    \caption{Overview of the bilevel optimization setup to find the best reward shape and the optimal policy. The pipeline involves an outer and inner training loop, where the outer step optimizes the Bayesian optimization model and samples the weights for our reward shape. The inner step optimizes the classic RLHF objective.}
    \label{fig:method}
\end{figure*}

\section{Preliminaries}

In this section, we first introduce the classic RLHF approach and build the relevant background for explainability and Bayesian optimization for our method. 

\subsection{Token-level MDP for RLHF}
We cast RL for language modeling in the sequential decision-making setting as a Markov Decision Process (MDP) by tuple $\mathcal{M} = (\mathcal{S}, \mathcal{A}, \mathcal{P}, \gamma, r)$. Here the \emph{state space} $\mathcal{S}$ contains states $s_t=\{x,\,y_{0:t}\}$, where $x$ denotes the input prompt and $y_{0:t}$ the sequence produced up to token $t$. The \emph{action set} $\mathcal{A}$ consists of next‑token choices $a_t$ sampled from the model’s token distribution. Transition dynamics are governed by the kernel $\mathcal{P}$, which assigns a probability 
$P(s' \mid s,a)$ to moving from state $s$ to $s'$ after taking action $a$. A scalar \emph{discount factor} $\gamma\in[0,1)$ down‑weights future returns, and the $r(s_t,a_t)$ is the reward.

In the RLHF pipeline, a model is first trained over a preference dataset $\mathcal{D} = \{(x, y_c, y_r)\}$ consisting of prompts $x$ and a pair of chosen and rejected responses $y_c, y_r$ via supervised fine-tuning (SFT) followed by a two-stage process (1) reward modeling and (2) RL training. From the SFT model, the reward function is modeled under a preference model encoded over the \textbf{sentence level} and trained as a binary classification problem via maximum likelihood: 

\begin{equation*}
    P(y_c \succ y_r) = \frac{\exp \{ r(x, y_c) \}}{{\exp \{r(x, y_c) \}} + {\exp \{r(x,y_r) \}}}
\end{equation*}

After learning the reward model, a language model is usually trained to optimize the reward via a policy gradient algorithm such as PPO \citep{schulman2017proximalpolicyoptimizationalgorithms} with KL-regularization to maximize the expected accumulated rewards. Here, since we optimize over a finite horizon, we set $\gamma = 1$ and consider the undiscounted return:
\begin{equation*}
    J^*(\pi_\theta) = \max_{\theta} \mathbb{E}_{a_t \sim\pi_\theta} \left [\sum_{t=0}^T r(s_t, a_t) - \beta \log \frac{\pi_\theta(a_t| s_t)}{\pi_{\text{ref}}(a_t| s_t)} \Bigg| s_0 \right ]
\end{equation*}

where $\pi_\theta$ is the language model agent and $\pi_{\text{ref}}$ is a reference policy, typically the SFT model.
In practice\footnote{We use the TRL \citep{vonwerra2022trl} implementation of PPO for our method}, the policy gradient optimizes a token-level value function given by rewards:

\begin{equation*}
    r(s_t, a_t) = \begin{cases}
        r(x,y) - \beta \log \frac{\pi_\theta(a_t | s_t)}{\pi_{\text{ref}}(a_t| s_t)} & \text{if } y \text{ is terminal}
        \\ - \beta \log \frac{\pi_\theta(a_t | s_t)}{\pi_{\text{ref}}(a_t | s_t)} & \text{ otherwise }
    \end{cases}
\end{equation*}

As a result, the signal received is very sparse \citep{zhong2025dpomeetspporeinforced, chan2024denserewardfreereinforcement} in which the model receives no relevant feedback about intermediate token generations beyond the KL-regularized term. Thus, it is of interest to understand how we can design rewards on the token level to provide quality, dense signals to make learning easier and potentially improve performance. 




\subsection{Explainability as Token-level Rewards}

We first introduce a general definition for estimating token-level information with explainability in the textual environment. To explain the predictions of a complex black-box function $f$ (i.e., our reward model), often a small linear model $g$ is used to locally approximate it concerning an input $x$ of the prompt and completion. Here, we consider only the family of \textit{additive feature attribution} functions  \citep{ribeiro2016whyitrustyou, lundberg2017shapley} such that the optimal policy with the original reward remains unchanged \citep{ng1999-policy-invariance}:
\begin{equation}
    g(z') = \phi_0 + \sum_{i=1}^M \phi_i z'_i
\end{equation}
In the text domain, we define $z'_i \in \{0, 1\}^M$ as a binary mask that masks the $i$-th token from the original input $x$ where $M$ is the number of tokens in the sequence such that $g(z')\approx f(h_x(z'))$ and $h_x$ is mapping a function that reconstructs $z'$ into a valid text input. The value of interest for each token is $\phi_i$, representing the feature effect or their marginal contribution to the prediction $f(x)$. To solve for the explainability score $\phi_i$ for each token, we optimize the local accuracy to sample $x$ by minimizing the penalized linear regression problem: 

\begin{equation}
   \mathcal{E} = \argmin_{g \in G} \mathcal{L}(f, g, \pi_{x}) + \Omega(g)
\end{equation}

where $G$ defines the set of local functions (i.e., small linear model), $\pi_{x}$ is a kernel transformation to the input (i.e., the SHAPley kernel in Eq. \ref{appendix:shap-kernel}), and $\Omega$ is a complexity penalty of $g$. Solving the regression problem gives us $\mathcal{E}$ or the set of explanation scores per token $\phi_i$ for a given sequence $x$. The form of Eq. 2 is general for any method of additive feature attribution functions, and we see later that formulating this as a reward-shaping function optimizes for the same policy as the original sparse reward objective in Section \textbf{3.1}. We detail a complete example of solving the explainability objective in Appendix \ref{appendix:a} to produce token-level scores.


\subsection{Bayesian Optimization for Reward Shaping} 
To optimize the token-level weights in our reward shaping function, we employ Bayesian Optimization (BO) as a zeroth-order method to efficiently search for high-reward configurations. BO aims to find the global maximizer $x^* = \arg\max_x f(x)$ of a black-box function $f$, which in our case represents the scalar reward produced by a shaped reward model under a given weight configuration. Since $f$ is expensive to evaluate, BO uses a surrogate model—typically a Gaussian Process (GP)—to approximate $f$ from past observations and suggest new candidate points.

In this work, it is sufficient to understand the GP as a distribution over a collection of function outputs mapping the observed function values $y = f(x)$ (here, the reward prediction) to a multivariate Normal: $f(x) \sim \mathcal{N}(\mu(x), \Sigma(x))$, where $\mu(x) = \mathbb{E}[y |x]$ is the mean function, and $\Sigma(x) = \text{Cov}[y|x]$ is the covariance function. To update the GP model, we can combine the observed values with the prior distribution to update the belief about $f$ and generate a posterior distribution, which we continue to do iteratively. 

After building the surrogate model, BO samples points by maximizing the expected utility\footnote{Formally, it is usually also referred to as the "reward" but we separate definitions to avoid confusion from RL rewards.} based on the current posterior under some acquisition function to sample areas of improvement and balance between exploration and exploitation to choose the next point. At each BO step, we use the surrogate to select a new candidate weight vector by maximizing an acquisition function $u$, which quantifies the expected utility (i.e., predicted improvement in validation reward) of sampling a new point: \begin{equation} 
w = \arg\max_{w} u(w \mid \mathcal{D}_{1:t-1}) \end{equation} where $\mathcal{D}_{1:t-1}$ is accumulated observations from the objective function. For $u$, we adopt log Noisy Expected Improvement \citep{ament2025unexpectedimprovementsexpectedimprovement}, which handles noisy evaluations well and is easy to optimize numerically. We implement BO using the \texttt{Ax} API \citep{bakshy2018}, which manages the optimization loop for the surrogate GP and acquisition functions.
\section{Explainable Reward Shaping as a Bilevel Optimization}
In this section, we discuss the optimization setup and the preliminary algorithm to find the best reward shape from each token-level score. First, we discuss the formulation of reward shape to compute a dense reward from locally approximating token-level rewards via explainability methods and obtain more fine-grained feedback. Next, we discuss the bilevel optimization problem, where we model BO and the classic RLHF problem as a nested objective to optimize both the reward shape and the policy. 

\subsection{Computing the Reward Shape}

At each step of the optimization procedure, after sampling the weights $w$ from the upper-level step, we shape the reward by computing a weighted linear combination of all token-level scores and broadcast the reward over the sequence:
\begin{equation}\label{eq:shaping}
\begin{aligned}
    r'(s, a) &= W^\top \boldsymbol{\mathcal{E}} \cdot r(s,a) = w_1 \cdot \mathcal{E}_{\text{SHAP}} \cdot r(s, a) + w_2 \cdot r(s,a)
\end{aligned}
\end{equation}
where $\boldsymbol{\mathcal{E}} = [\mathcal{E}_{\text{SHAP}}, 1]$ is the matrix of token-level explanations (not limited to just SHAP); $W$ is the vector of weights sampled from the Bayesian optimization step. Importantly, we sample each $w_i$ from the BO model such that $\sum_{i} w_i = 1$ to consider the convex combination of all token-level scores over the reward.
In practice, we compute the dense reward by first taking the softmax over the token-level scores to get a probability vector over the sequence. Then, following \cite{chan2024denserewardfreereinforcement}, we distribute the sparse reward or assign credit to each token by broadcasting it over the sequence. As a result, the scores at intermediate states are assigned to provide fine-grained feedback for the next token generated transitioning from state $s \to s'$.


Additionally, we show that any explainability method in the family of additive feature attribution functions follows policy invariance and can be formulated as a potential-based shaping function \citep{ng1999-policy-invariance}. This ensures that adding a reward for transitions between states does not yield a suboptimal policy but better guides the learning agent towards the same optimal policy. We detail a complete proof in Appendix \ref{appendix:policy-invariance}.


\begin{figure}
    \centering
    \includegraphics[width=0.85\linewidth
        , trim={4.5cm 3.75cm 4.5cm 2.75cm}, clip
        ]{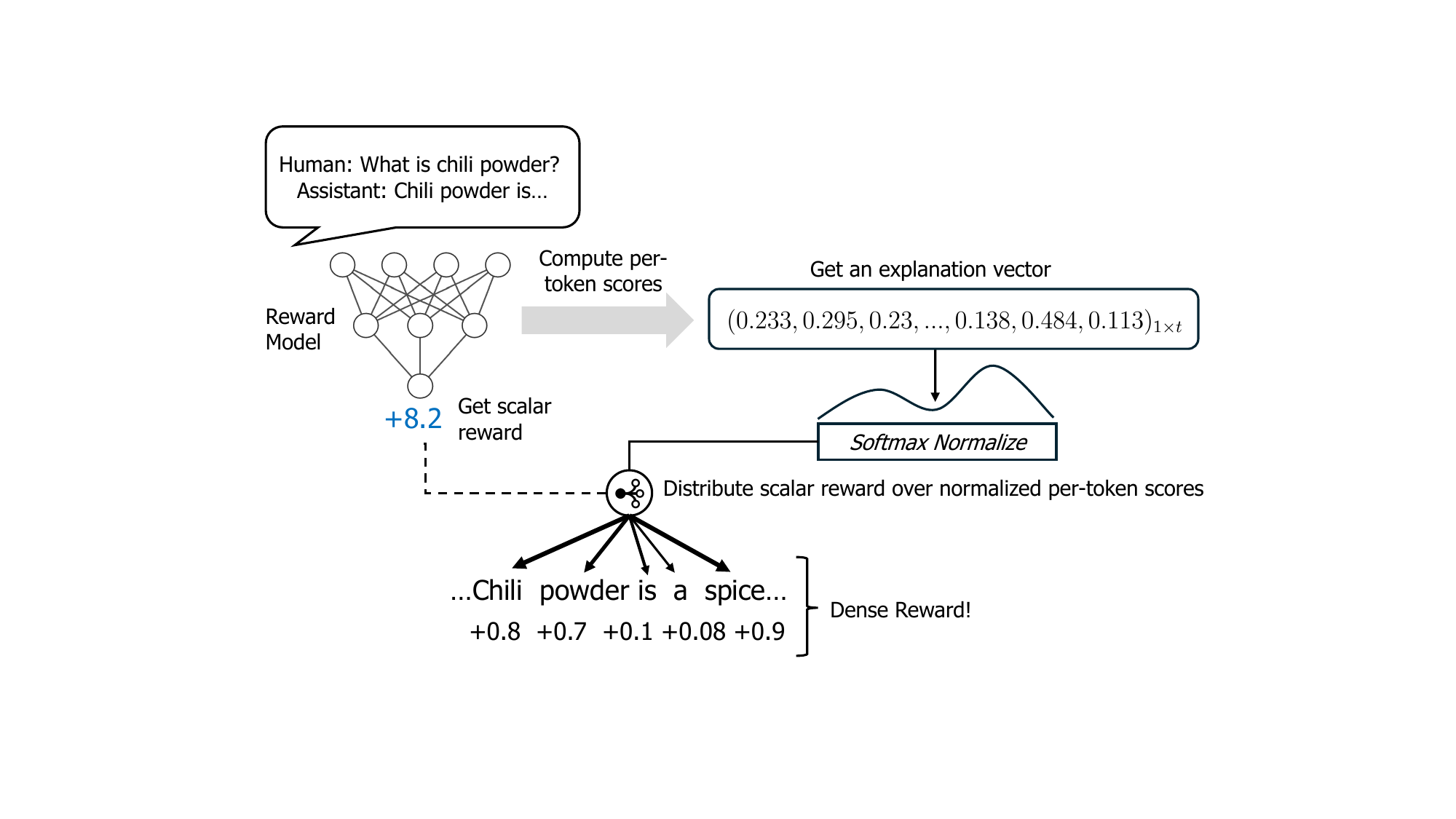}
    \caption{Redistribution sequence of the scalar reward prediction over the explanation feature attributions after softmax normalization. A darker red highlights a much stronger positive contribution, while a deeper blue indicates a more negative contribution.\vspace{-0.2cm}}
    \label{fig:reward_redistribute}
\end{figure}

\subsection{BO and RLHF as Nested Problems}

Conventionally, the lower level is an auxiliary problem that helps solve the upper-level problem \citep{zhang2023introductionbileveloptimizationfoundations}.In contrast, here, we are equally interested in the solution to the lower-level problem as a result of the solution retrieved from the upper-level one. In other words, we have the following bilevel optimization setup:
\begin{equation}
\begin{gathered}
    \max_{w} f(w; \pi_\theta^*(w)) \quad 
    \text{s.t.} \quad 
    \pi_\theta^*(w) = \argmax_{\pi_\theta \in \Pi} 
    \mathbb{E}_{a \sim\pi_\theta} \left [r'(s,a) - \beta \log \frac{\pi_\theta(a| s)}{\pi_{\text{ref}}(a| s)} \right ]
\end{gathered}
\end{equation}
where we have the upper-level optimization problem maximizing the acquisition function $f$ in each BO iteration to sample weights $w$ based on the validation set performance (i.e., average validation reward) of the best-performing policy $\pi^*_\theta$. In the lower-level optimization step, we have the standard RLHF problem to find the best policy $\pi_\theta^*$ given the weights for reward shaping from the upper-level step. In principle, our optimization procedure can be more concretely understood as a hyperparameter tuning problem, as we do not explicitly parameterize the reward transformation, but where the outer step looks to find the best set of weights $w$ (or hyperparameters) to compute the reward shape for the optimal policy. We outline a practical implementation of our method in Algorithm \ref{alg:bo-ppo}. 


\begin{table*}[t!]
\small
\centering
\begin{tabular}{l c cc ccc}
\toprule
  \multicolumn{1}{l}{Models} & \# & \multicolumn{2}{c}{HH-RLHF}  & \multicolumn{3}{c}{Ultrafeedback} \\
\cmidrule(lr){3-4} \cmidrule(lr){5-7}
& & Score & Win (\%) & Score & Win (\%) & MTBench \\
\midrule
\textbf{Baselines}\\
\textrightarrow  SFT      &   -   &        5.40      &     --      &  4.34            &          --            &           7.2   \\
\textrightarrow  RLHF with Sparse (original)      &   - &      5.94        &       48.7    &    4.35          &        57.64           &        7.17      \\
\midrule
\textbf{Dense RLHF}\\
\textrightarrow  Attention \citep{chan2024denserewardfreereinforcement}   &  1 &      6.24        &     51.40      &       4.42       &      \underline{61.40}      &        7.28      \\
\textrightarrow  SHAP*    & 1&        5.91      &     53.95      &     4.46          &          53.81   &     \textbf{ 7.81 }       \\
\textrightarrow  LIME*   &  1 &      5.86        &  54.36         &    4.46          &       59.20     &      7.28      \\
\midrule
\textbf{Dense RLHF with BO}\\
\textrightarrow BO-SHAP & 1 &    5.99        &    55.75      &   4.37     &     60.82       &           7.37   \\ 
\textrightarrow BO-LIME  & 1 &    6.26         &      \textbf{57.62}    &      4.37        &      56.73       &      7.59        \\
\noalign{\vskip 0.5ex}
\cdashline{1-7}
\noalign{\vskip 0.5ex}
\textrightarrow BO-SHAP-Attn  & 2 &   \underline{6.47}         &      \underline{56.90}      &     \underline{4.48}        &    51.08        &      \underline{7.71}        \\
\textrightarrow BO-LIME-Attn  &  2 &    6.30      &   54.09   &       4.45       &     55.03       &       7.03       \\
\textrightarrow BO-SHAP-LIME & 2 &      \textbf{ 6.58}     &      56.28      &        \textbf{4.51}     &    \textbf{64.82}        &     7.61      \\
\noalign{\vskip 0.5ex}
\cdashline{1-7}
\noalign{\vskip 0.5ex}
\textrightarrow BO-SHAP-LIME-Attn  & 3 &     5.74        &    51.00      &      4.44        &    59.64        &        7.38      \\
\bottomrule 
\end{tabular}
\caption{Comparison of HH-RLHF  (Only Helpful) and Ultrafeedback. Winrate is calculated via AlpacaEval, with the reference model being the SFT model's generation. 
* indicates strong baselines that we have implemented. 
\# means the number of dense reward types used.
Score refers to the average test set reward and Win 
(\%) refers to the length-controlled win-rate against the baseline SFT generations. We use \texttt{gpt4-turbo} as the judge for both AlpacaEval2 and MTBench. Bolded and underlined numbers highlight the best performance.}
\label{tab:main_results}
\vspace{-0.4cm}
\end{table*}

\section{Experiments}
\textbf{Tasks}.
In this section, we empirically verify the effectiveness of our method on downstream tasks. Namely, we train each model on two different \textbf{single-turn dialogue} datasets: 1) HH-RLHF (helpfulness) \citep{bai2022traininghelpfulharmlessassistant} and 2) Ultrafeedback \citep{cui2024ultrafeedbackboostinglanguagemodels} datasets. For HH-RLHF, we utilize the OpenLLaMA family of models with an instruction-tuned 7B model as the SFT model and the 3B parameter reward model from \cite{dong2023raftrewardrankedfinetuning} following the setup in \citep{chan2024denserewardfreereinforcement}. For Ultrafeedback experiments, we utilize \textsc{LLaMa}-3.2-Instruct 3B \citep{grattafiori2024llama3herdmodels} as the SFT model and fine-tune \textsc{LLaMa}-3.2-Instruct 1B on  Ultrafeedback preferences as the reward model. 

To verify the effectiveness of each method, we consider first the average holistic reward predicted by the reward model over the test splits for each dataset. We also evaluate on open benchmarks such as AlpacaEval-2 \citep{dubois2025lengthcontrolledalpacaevalsimpleway} and MTBench \citep{zheng2023judgingllmasajudgemtbenchchatbot} to examine whether our methods lead to a better local optimum versus optimizing with only the sparse reward while also avoiding reward overfitting.

\textbf{Baselines} To measure the improvement of our method, we consider several different baselines: 1) the SFT model to calibrate improvements, 2) RLHF on the sparse reward, 3) attention-based credit \citep{chan2024denserewardfreereinforcement}, and 4) each individual explanation method before optimizing the reward shape with Bayesian Optimization. 

\textbf{Training setup} For every experiment, we adopt PPO as our method's principle policy gradient algorithm. Due to the computational complexity of PPO and to run a sufficient number of Bayesian trials, we must compromise training samples per trial to avoid inducing too much computation overhead. Hence, instead of running over the whole dataset for each BO trial, we run $m = 25$ trials for the Bayesian Optimization step and batch each trial into a small number of training epochs (i.e., 10) of batch size 8 (80 samples) for a total of 2000 samples. To initially sample weights for reward shaping, we employ Sobol sampling for the first five trials to build a prior for the Gaussian Process model and then sample from the GP model for the remaining trials.

After each BO iteration, we evaluate the trained model on the validation split \footnote{We construct the validation set by splitting the training dataset 90\%/10\%.} and take the average validation reward as the utility to update the BO surrogate function. To continue training, we employ model checkpointing to resume training from the best checkpoint instead of running a complete training cycle from scratch. Thus, we run at most only twice over the dataset $\mathcal{D}_{1:n}$ (once for BO, once for full PPO). To ensure randomness, we randomly sample a subset from the training dataset at each Bayesian trial to train PPO. We also randomly sample a subset from the validation set at each validation step. We detail our model setups and hyperparameters and discuss the computational complexity in  Appendix \ref{appendix:wall-time} and \ref{appendix:impl-algorithm}. 

\begin{figure}[t]
\centering
\subfigure[SHAP vs. BO variants]
{\label{fig:4a}\includegraphics[width=0.3\linewidth
                , trim={0.0cm 0.0cm 0.0cm 0.0cm}
                , clip
            ]{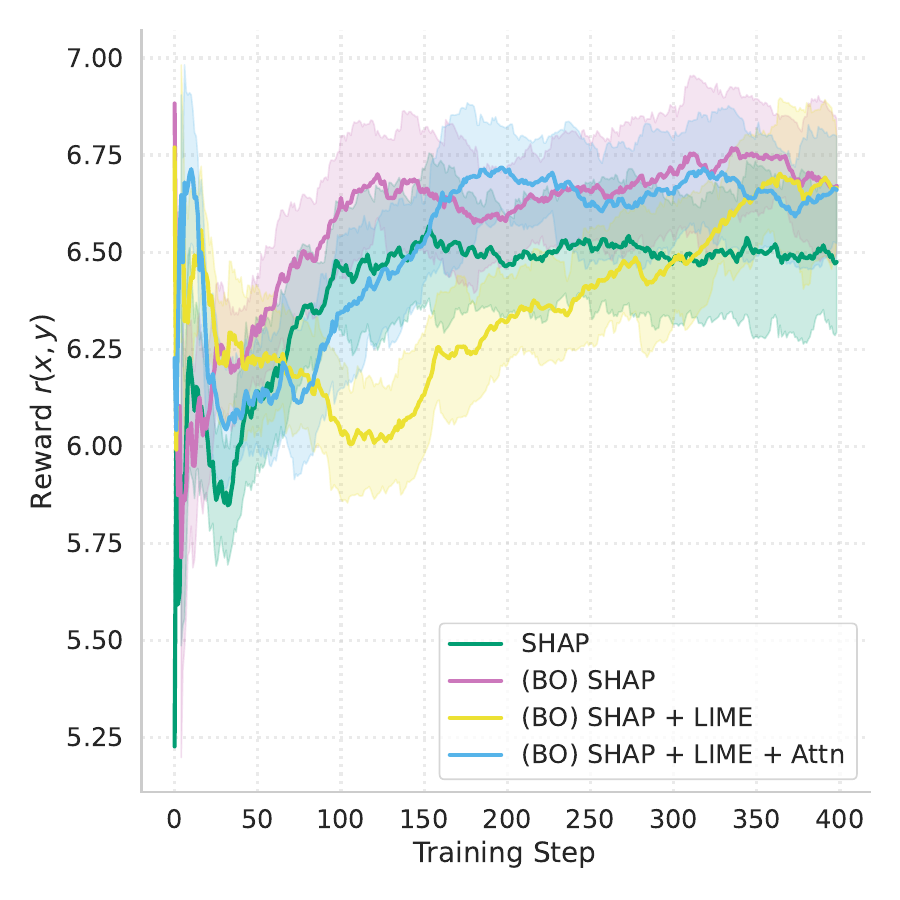}}
\hfill
\subfigure[Validation reward after each BO trial.]
{\label{fig:4b}\includegraphics[width=0.3\linewidth
                , trim={0.0cm 0.0cm 0.0cm 0.0cm}
                , clip
            ]{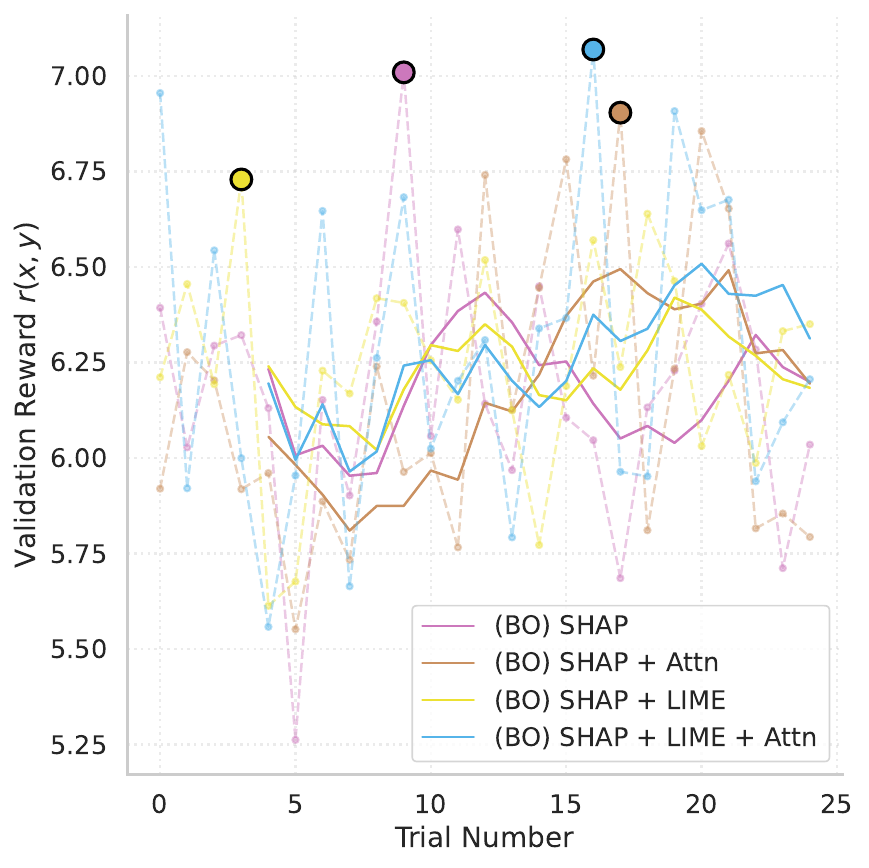}}
\hfill
\subfigure[Mean dense reward shape]
{\label{fig:4c}\includegraphics[width=0.3\linewidth
                , trim={0.0cm 0.0cm 0.0cm 0.0cm}
                , clip
            ]{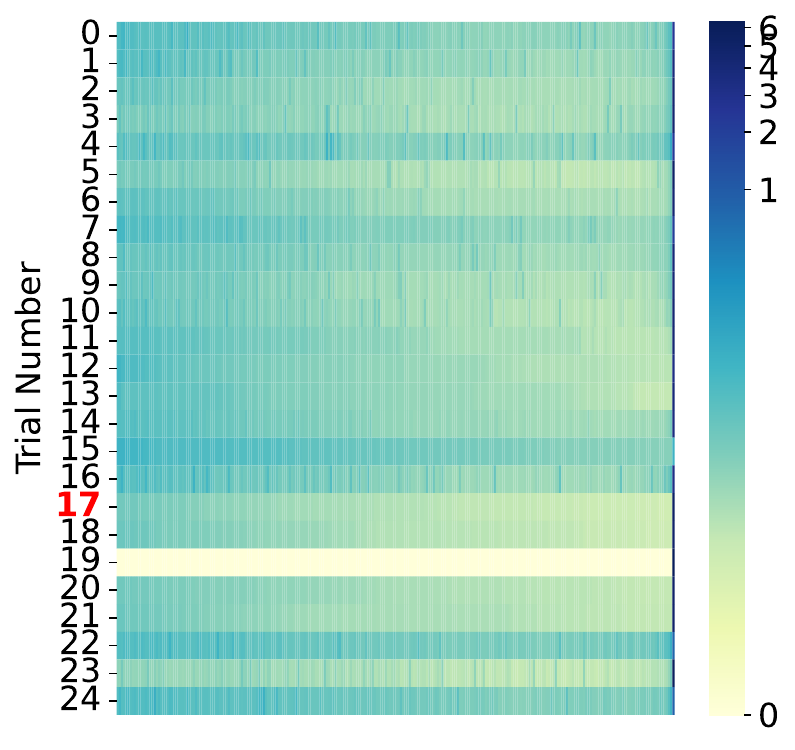}}
\caption{\textbf{(Left)} The mean training reward per timestep with increasing BO dimensionality. \textbf{(Middle)} The mean validation reward over each BO trial. The highlighted dots indicate the best validation reward received at trial $n$. \textbf{(Right)} The average dense reward attribution over each trial for SHAP + ATTN. The highlighted row indicates the shape in trial 17 that received the highest validation reward.
}
\end{figure}
\vspace{-0.2cm}

\begin{wrapfigure}{R}{0.6\textwidth}
\vspace{-8mm}
\centering 
\subfigure[Explanation-based rewards training curves vs. baselines]
{\label{fig:3a}\includegraphics[width=0.48\linewidth
                , trim={0.0cm 0.0cm 0.0cm 0.0cm}
                , clip
            ]{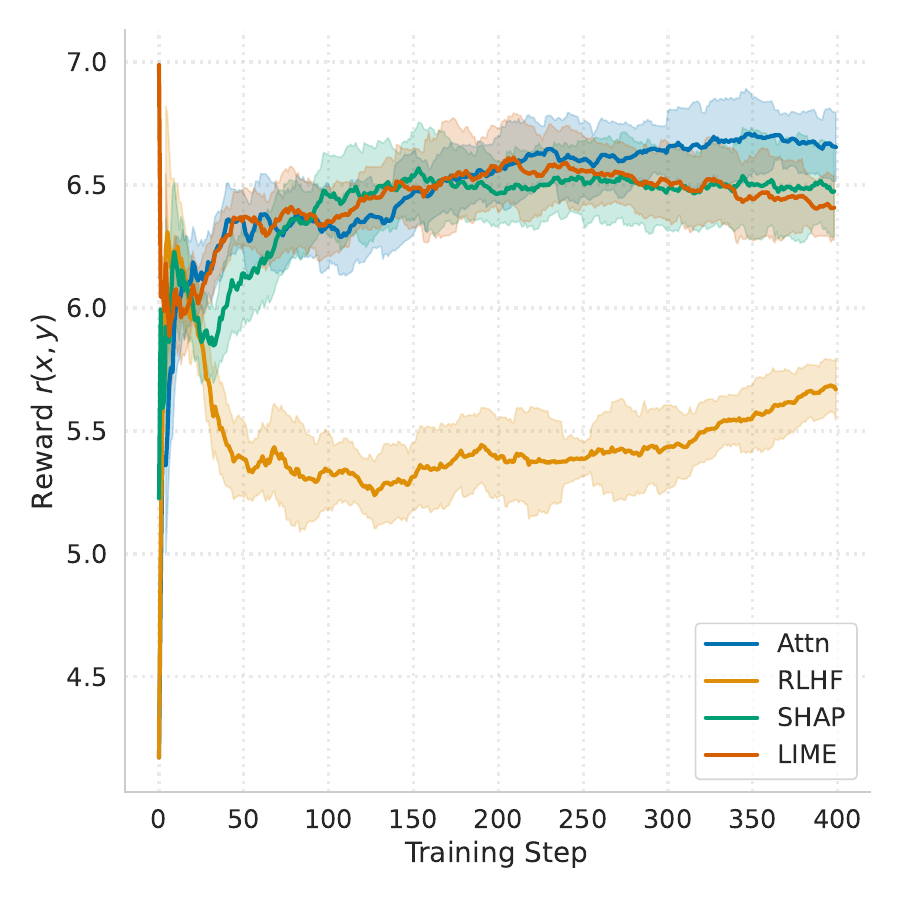}}
\subfigure[PPO Value loss]
{\label{fig:3b}\includegraphics[width=0.48\linewidth
                , trim={0.0cm 0.0cm 0.0cm 0.0cm}
                , clip
            ]{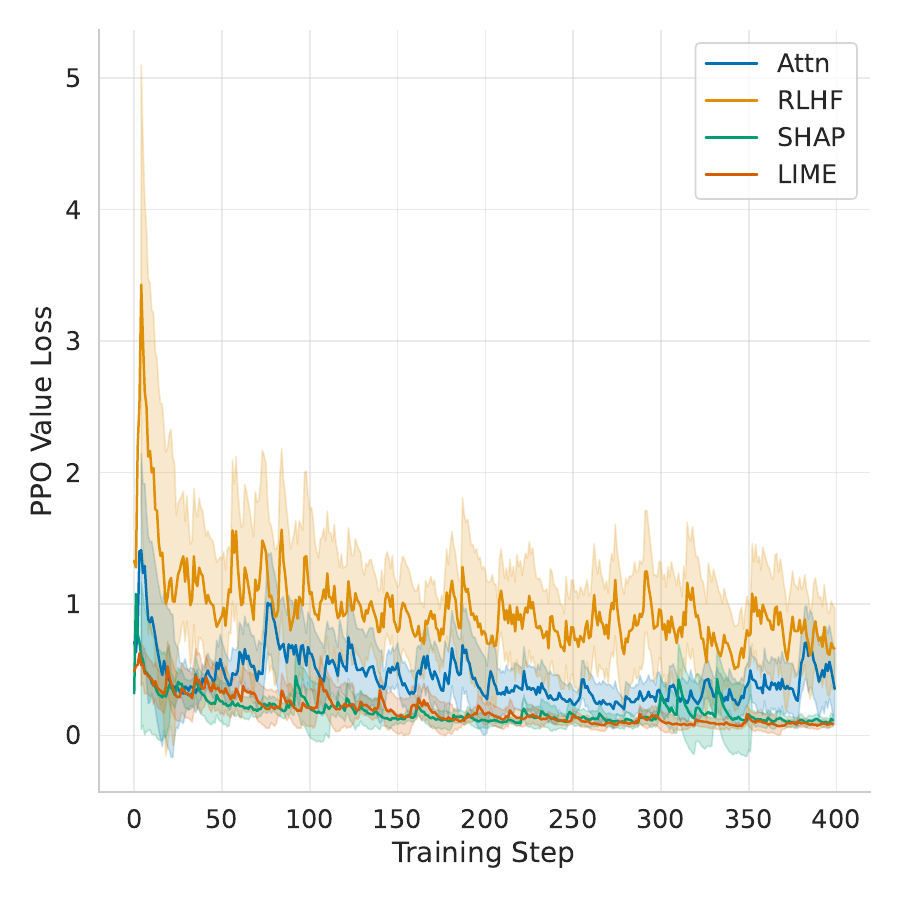}} \vspace{-3mm}
\caption{
\textbf{Helpfulness}. \textbf{(Left)} The average training reward per timestep. \textbf{(Right)} The average value head loss per timestep. The shading represents the standard error ($95\%$ confidence interval) as training progresses.
\vspace{-5mm}
}
\end{wrapfigure}

\subsection{How well do explanation-based rewards optimize the RLHF objective?}
We first analyze the impact of explanation-based rewards in optimizing the vanilla RLHF objective. First, we do not expect to see any significant improvements from explanation methods alone. As we formulate explanation scores as a potential-based reward shaping transformation, we are not guaranteed to find better local optima \cite{ng1999-policy-invariance}, but a carefully chosen reward shape could improve the learning complexity of reinforcement learning \citep{fu2025rewardshapingmitigatereward, gupta2022unpackingrewardshapingunderstanding}. Foremost, we are interested in whether explainability can provide any meaningful improvements or optimizations over simply considering the sparse, scalar reward. 

\paragraph{Credit assignment with explainability helps exploration} Before applying any Bayesian optimization, we manually select a weighting combination, placing higher emphasis on the token-level scores with $w = 0.8$ following \cite{chan2024denserewardfreereinforcement}. Overall, explanation-based methods achieve a high average reward relatively early during training that is competitive with the attention-based baseline observed in Figure \ref{fig:3a}. Furthermore, explanation-based methods achieve more stable updates than sparse rewards from Figure \ref{fig:3b}. In particular, the dense rewards from explainability drastically reduce the PPO value head loss, signaling that we have a good approximation of states to long-term returns for more stable policy updates during training.

\paragraph{Explanation-based rewards avoid reward overfitting} One of the main risks associated in reinforcement learning from human feedback is reward overfitting. This is particularly undesirable since the reward model is a proxy for human preferences, and overoptimizing its value can hinder ground truth performance \citep{gao2022scalinglawsrewardmodel}. To study this relationship, we observe the performance of baselines and our methods on open benchmarks for instruction-following and multi-turn dialogue capabilities, which contain data samples not strictly in distribution to our training datasets. From Table \ref{tab:main_results}, we see that explanation-based methods do not substantially outperform baselines on the HH-RLHF test split; however, they achieve a markedly higher win rate on open benchmarks, indicating that we are not particularly overfitting to the reward function and also maintaining general performance. Notably, we see the sparse reward approach underperforms relative to the SFT model on AlpacaEval2 for the HH-RLHF models. We suspect this stems from overfitting on the reward function, which can degrade performance on data splits not directly related to the helpfulness split. To verify this, we examine the win rate against the SFT model only the helpfulness split of AlpacaEval2 and find the win rate improves, being preferred $56.34\% \pm4.03$ of the time. 

\begin{figure}[t]
\centering
{\includegraphics[width=\linewidth
                , trim={0.0cm 3.5cm 0.0cm 3.5cm}
                , clip
            ]{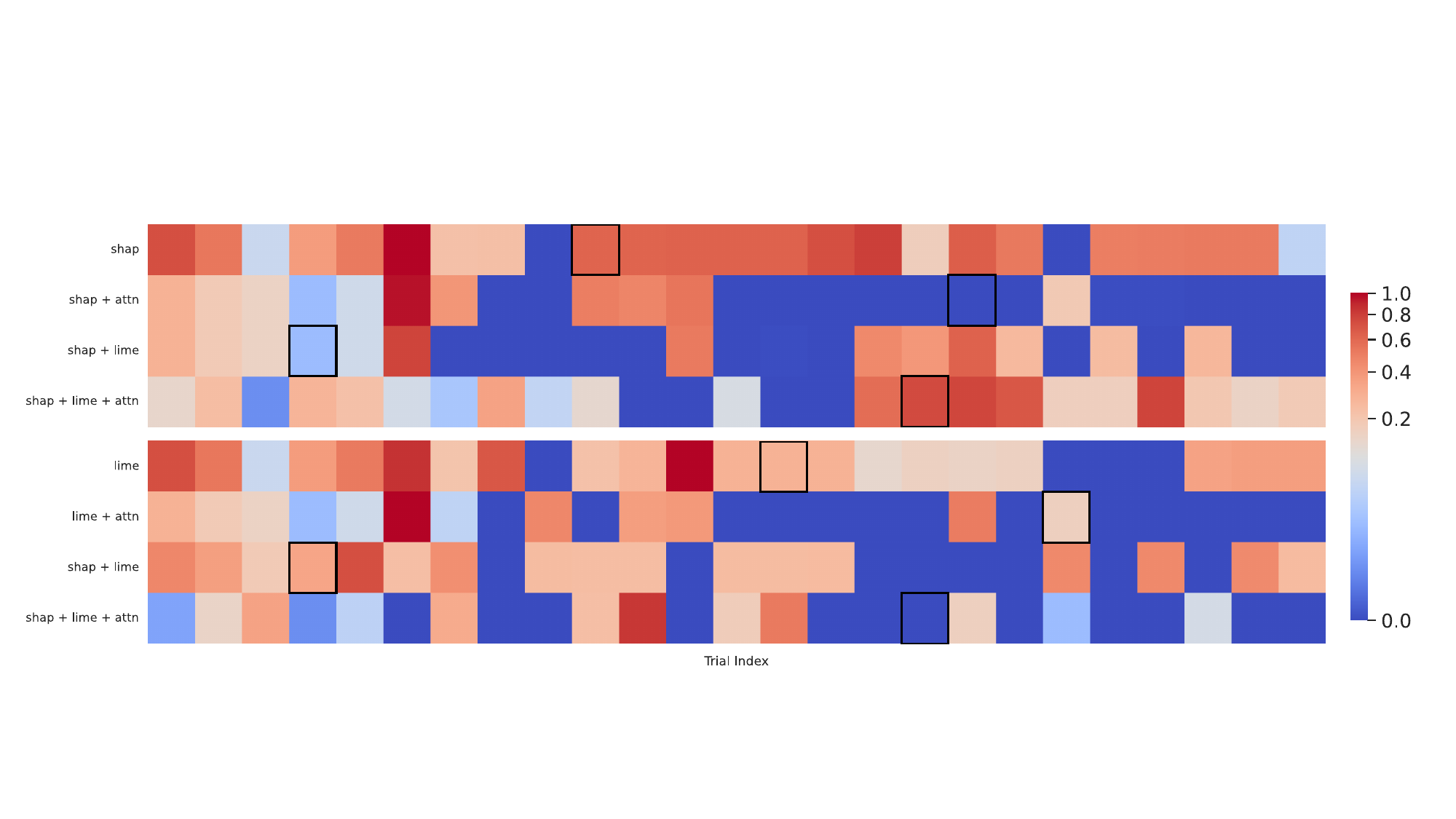}}
\vspace{-0.4cm}
\caption{\textbf{(Top)} The weight transition between trials for SHAPley scores. \textbf{(Bottom)} The weight transition between trials for LIME scores. The black boxes indicate the "best weights" sampled by the BO model.}
\label{fig:weight_distribution}
\end{figure}

\subsection{Does Bayesian optimization help balance token-level rewards better?}

Next, we evaluate the impact of Bayesian Optimization (BO) in sampling optimal weights to construct a new, dense reward function based on token-level interpretability scores. We consider several combinations of token-level scores derived from SHAP, LIME, and the attention map of the reward model. Specifically, our goal is to assess whether BO can effectively weigh these different sources of token-level importance to produce a more informative reward signal during policy optimization.

However, due to the limited number of trials available, we acknowledge that convergence of BO is not guaranteed, particularly given the high-dimensional and potentially non-convex search space of reward weight combinations \citep{loeppky2012choosing, snoek2012practicalbayesianoptimizationmachine}. The challenge of balancing exploration and exploitation in such a setting further complicates the optimization process, especially when evaluating noisy, non-stationary reward surfaces. Despite these limitations, Table \ref{fig:method} and Figure \ref{fig:4a} show that the weights sampled by BO generally lead to improved downstream task performance compared to baselines and that models guided by BO-optimized dense rewards tend to achieve even higher training rewards. 

\paragraph{Explore-Exploit Tradeoff} To better understand BO's behavior, we analyze the evolution of weight distributions and dense reward shape across trials in Figure \ref{fig:weight_distribution} and \ref{fig:4c}. In particular, we highlight that the final combination incorporating all scores with an input dimension of $d=4$ does not outperform simpler combinations. In particular, this is expected as by increasing the search space complexity, we should also increase the number of trials; however, due to the computation constraints, we keep the number of trials constant.

Ideally, if BO had successfully identified optimal weights, we would expect to see a lower performance bound on any sub-combination of scores; that is, more scores should not hurt performance if adequately weighted. However, the observed degradation in performance for $d=4$ implies that BO may not have sufficiently explored or exploited the reward space, potentially due to early convergence to suboptimal regions. Across all trials in Figure \ref{fig:weight_distribution}, we see that after Sobol sampling from the first five trials, BO tends to consistently switch between exploration and exploitation, suggesting that just more exploration is needed.

\begin{figure}[t]
\centering
\includegraphics[width=\linewidth
                , trim={0.0cm 0.0cm 0.0cm 0.0cm}
                , clip
            ]{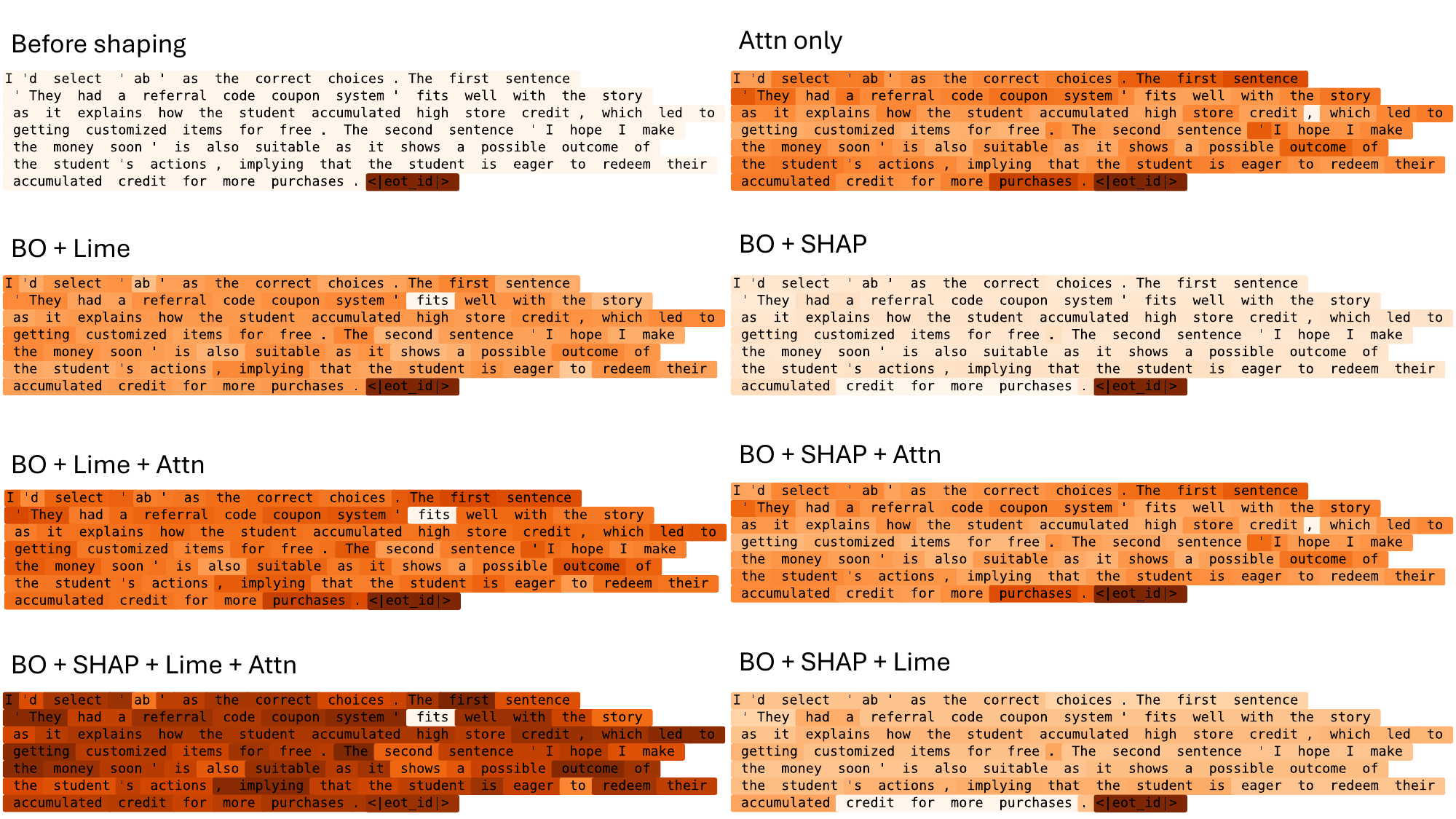}
\vspace{-0.5cm}
\caption{\textbf{(Ultrafeedback)} The top left represents the baseline per-token reward without shaping. The color of each token represents the reward received, with a darker color representing a higher proportion of the reward assigned. A more uniform coloring indicates a more uniform assignment of the scalar reward to each token, while a contrasting light/dark coloring indicates a more skewed assignment.}
\label{fig:shaping_heat}
\vspace{-0.5cm}
\end{figure}

\paragraph{BO and Credit Assignment} Additionally, we inspect the before and after of the reward shape on a representative example from the Ultrafeedback \citep{cui2024ultrafeedbackboostinglanguagemodels} dataset, where the model is asked to choose two sentences out of three that seamlessly connect with an unfinished story. Here, in Figure \ref{fig:shaping_heat}, we show several variants of the shaping method on a completion from the SFT model where the tokens are colored proportionally to the reward they received after computing the shaping weights. In particular, almost all shaping methods still place importance on the terminating token. We suspect that since the reward model is trained to backpropogate on the terminal token, we cannot ignore it in order to maximize the reward function best. Specifically, we see the combination of SHAP + LIME, which highlights the explanation of the model's rationale for choosing the first and second sentences. We similarly see this for the $d=4$ combination with all token-level scores. However, we see a more uniform assignment of scores as the scalar reward was given zero weight in the final shaped score. 

\paragraph{Limitations of Current Explainability Methods}
While Bayesian Optimization (BO) can leverage complementary effects from multiple explainability signals, individual mechanistic interpretability methods, such as a model’s intrinsic attention, remain fundamentally suboptimal.
As shown in Figure \ref{fig:shaping_heat}, these methods often fail to align precisely with human perception of relevant features \cite{bereska2024mechanistic}. A promising direction for future research is to incorporate more fine-grained, human-aligned signals, potentially sourced from explicit token-level annotations \cite{hayati2021does} or cognitive indicators like eye movements when reading text \cite{de2022comparative}.



\section{Related Work}
\textbf{Reward Shaping} is a technique in reinforcement learning (RL) that supplements traditionally sparse rewards with more informative reward signals \citep{hurewardshapinga}. Early work in potential-based reward shaping (PBRS) demonstrates that reward shaping can preserve the optimal policy while reducing training time or introducing domain knowledge \citep{hurewardshapinga, caorewardshapingc}. However, challenges still arise because it is often difficult to guarantee that the additional reward is helpful for any given task \citep{hurewardshapinga}. 
Additionally, we note that we are not the first to inspect reward shaping or RLHF under a bi-level optimization lens. For example, bi-level optimization of parameterized reward shaping (BiPaRS) \citep{hurewardshapinga} adaptively utilizes a given reward function, optimizing both the policy itself, as well as the shaping weight for the reward function. \cite{shen2024principledpenaltybasedmethodsbilevel} Also inspects RLHF as a bilevel formulation under a penalty function. Other reward shaping methods focus on adaptability to overcome limitations, such as AlphaPO \citep{guptarewardshapingb} builds on simple policy optimization (SimPO) by introducing a new scalar parameter $\alpha$, demonstrating that the shape of the reward function can alter the likelihood displacement. 

\textbf{Token-Level RLHF} Other methods aim to improve reward signals by distributing the reward at the token-level, hoping to capture more nuance in particular words or phrases \citep{yoontokenlevelc} or discourses \citep{kim2025align} in the generation ; examples include Token-Level Continuous Reward (TLCR), which uses a token-level preference discriminator \citep{yoontokenlevelc}, RLMEC which leverages minimum editing constraints to produce token-level supervision \citep{chentokenleveld}, and DRLC, which employs LLMs to identify dense positive and negative labelings within a response \citep{caotokenleveli}. Other work has also explored using LLMs themselves to provide the dense reward signal \citep{huangtokenlevelg}. A common theme with these works is training contrastively: an initial output is compared with a revised output, and the fine-grained reward signal is achieved by comparing the differences \citep{guotokenlevele}. Methods such as token-level PPO (TPPO) \citep{ouyangtokenlevelf} and token-level DPO (TDPO) \citep{zengtokenlevelb} further support the approach of token-level fine-tuning. 

\textbf{Bayesian Optimization} (BO) has also emerged to augment LLM training and fine-tuning \citep{austinboa, liubob, yangboc, gaobod, agarwalboe, kristiadibof, opsahlbog, cheboh}. BO functions by iteratively estimating the optimal posterior predictive distribution by adding new information/candidates \citep{agarwalboe}. Some recent work uses BO with LLMs to perform selection across a discrete search space \citep{kristiadibof, opsahlbog}. Other work investigates using BO to optimize prompts for LLMs, such as BOPRO \citep{agarwalboe}, InstructZero \citep{cheboh}, leveraging the Bayesian Optimization to explore the prompt search space. BO can also be leveraged in pre-training, by finding the optimal weighting for checkpoint merging \citep{liubob}.

BO can also be leveraged for uncertainty estimation \citep{yangboc}, bias mitigation (via win rate calibration) \citep{gaobod}, and guided query generation \citep{austinboa}. These tasks ultimately benefit from BO’s robustness to noise \citep{opsahlbog} by calibrating for uncertainty in the optimization process.

\section{Conclusion}

In this work, we show that explainability offers an effective way to reshape rewards for token-level credit assignments. Furthermore, our findings suggest that BO is a promising tool for learning and balancing different sources of token-level information for reward shaping. As a combination, we show that explainable dense reward shapes optimized under a Bayesian framework positively impact RL training compared to sparse rewards by accelerating learning and providing more stable updates to the value function. Furthermore, properly attributing tokens via interpretable methods also improves performance over the sparse reward on downstream tasks while also theoretically being policy invariant to stay faithful to the original reward model.   

In the future, we could consider incorporating contextual information into the BO weights to dynamically shape token-level rewards. For instance, we could learn a set of high-dimensional weights corresponding to the token-embeddings of a policy model's generation to dynamically compute the reward shape instead of the current statically weighted reward shape.

\section*{Acknowledgements}

This work was mainly supported by the research gift from Grammarly and UMN Data Science Institute (DSI) Seed Grant. We also thank Minnesota NLP group members for providing us with valuable feedback and comments on the initial draft.

\bibliography{colm2025_conference}

\begin{thebibliography}{59}
\providecommand{\natexlab}[1]{#1}
\providecommand{\url}[1]{\texttt{#1}}
\expandafter\ifx\csname urlstyle\endcsname\relax
  \providecommand{\doi}[1]{doi: #1}\else
  \providecommand{\doi}{doi: \begingroup \urlstyle{rm}\Url}\fi

\bibitem[Agarwal et~al.(2025)Agarwal, Arivazhagan, Das, Swamy, Khosla, and Gangadharaiah]{agarwalboe}
Dhruv Agarwal, Manoj~Ghuhan Arivazhagan, Rajarshi Das, Sandesh Swamy, Sopan Khosla, and Rashmi Gangadharaiah.
\newblock Searching for optimal solutions with {LLM}s via bayesian optimization.
\newblock In \emph{The Thirteenth International Conference on Learning Representations}, 2025.
\newblock URL \url{https://openreview.net/forum?id=aVfDrl7xDV}.

\bibitem[Ament et~al.(2025)Ament, Daulton, Eriksson, Balandat, and Bakshy]{ament2025unexpectedimprovementsexpectedimprovement}
Sebastian Ament, Samuel Daulton, David Eriksson, Maximilian Balandat, and Eytan Bakshy.
\newblock Unexpected improvements to expected improvement for bayesian optimization, 2025.
\newblock URL \url{https://arxiv.org/abs/2310.20708}.

\bibitem[Austin et~al.(2024)Austin, Korikov, Toroghi, and Sanner]{austinboa}
David Austin, Anton Korikov, Armin Toroghi, and Scott Sanner.
\newblock Bayesian optimization with llm-based acquisition functions for natural language preference elicitation.
\newblock In \emph{Proceedings of the 18th ACM Conference on Recommender Systems}, RecSys '24, pp.\  74–83, New York, NY, USA, 2024. Association for Computing Machinery.
\newblock ISBN 9798400705052.
\newblock \doi{10.1145/3640457.3688142}.
\newblock URL \url{https://doi.org/10.1145/3640457.3688142}.

\bibitem[Bai et~al.(2022)Bai, Jones, Ndousse, Askell, Chen, DasSarma, Drain, Fort, Ganguli, Henighan, Joseph, Kadavath, Kernion, Conerly, El-Showk, Elhage, Hatfield-Dodds, Hernandez, Hume, Johnston, Kravec, Lovitt, Nanda, Olsson, Amodei, Brown, Clark, McCandlish, Olah, Mann, and Kaplan]{bai2022traininghelpfulharmlessassistant}
Yuntao Bai, Andy Jones, Kamal Ndousse, Amanda Askell, Anna Chen, Nova DasSarma, Dawn Drain, Stanislav Fort, Deep Ganguli, Tom Henighan, Nicholas Joseph, Saurav Kadavath, Jackson Kernion, Tom Conerly, Sheer El-Showk, Nelson Elhage, Zac Hatfield-Dodds, Danny Hernandez, Tristan Hume, Scott Johnston, Shauna Kravec, Liane Lovitt, Neel Nanda, Catherine Olsson, Dario Amodei, Tom Brown, Jack Clark, Sam McCandlish, Chris Olah, Ben Mann, and Jared Kaplan.
\newblock Training a helpful and harmless assistant with reinforcement learning from human feedback, 2022.
\newblock URL \url{https://arxiv.org/abs/2204.05862}.

\bibitem[Bakshy et~al.(2018)Bakshy, Dworkin, Karrer, Kashin, Letham, Murthy, and Singh]{bakshy2018}
Eytan Bakshy, Lili Dworkin, Brian Karrer, Konstantin Kashin, Ben Letham, Ashwin Murthy, and Shaun Singh.
\newblock Ae: A domain-agnostic platform for adaptive experimentation.
\newblock In \emph{NeurIPS Systems for ML Workshop}, 2018.
\newblock URL \url{http://learningsys.org/nips18/assets/papers/87CameraReadySubmissionAE%20-%20NeurIPS%202018.pdf}.

\bibitem[Bereska \& Gavves(2024)Bereska and Gavves]{bereska2024mechanistic}
Leonard Bereska and Efstratios Gavves.
\newblock Mechanistic interpretability for ai safety--a review.
\newblock \emph{arXiv preprint arXiv:2404.14082}, 2024.

\bibitem[Cao et~al.(2024{\natexlab{a}})Cao, Shu, Yu, Zhu, Wichers, Liu, and Meng]{caorewardshapingc}
Meng Cao, Lei Shu, Lei Yu, Yun Zhu, Nevan Wichers, Yinxiao Liu, and Lei Meng.
\newblock Enhancing reinforcement learning with dense rewards from language model critic.
\newblock In Yaser Al-Onaizan, Mohit Bansal, and Yun-Nung Chen (eds.), \emph{Proceedings of the 2024 Conference on Empirical Methods in Natural Language Processing}, pp.\  9119--9138, Miami, Florida, USA, November 2024{\natexlab{a}}. Association for Computational Linguistics.
\newblock \doi{10.18653/v1/2024.emnlp-main.515}.
\newblock URL \url{https://aclanthology.org/2024.emnlp-main.515/}.

\bibitem[Cao et~al.(2024{\natexlab{b}})Cao, Shu, Yu, Zhu, Wichers, Liu, and Meng]{caotokenleveli}
Meng Cao, Lei Shu, Lei Yu, Yun Zhu, Nevan Wichers, Yinxiao Liu, and Lei Meng.
\newblock Enhancing reinforcement learning with dense rewards from language model critic.
\newblock In Yaser Al-Onaizan, Mohit Bansal, and Yun-Nung Chen (eds.), \emph{Proceedings of the 2024 Conference on Empirical Methods in Natural Language Processing}, pp.\  9119--9138, Miami, Florida, USA, November 2024{\natexlab{b}}. Association for Computational Linguistics.
\newblock \doi{10.18653/v1/2024.emnlp-main.515}.
\newblock URL \url{https://aclanthology.org/2024.emnlp-main.515/}.

\bibitem[Chan et~al.(2024)Chan, Sun, Holt, and van~der Schaar]{chan2024denserewardfreereinforcement}
Alex~J. Chan, Hao Sun, Samuel Holt, and Mihaela van~der Schaar.
\newblock Dense reward for free in reinforcement learning from human feedback, 2024.
\newblock URL \url{https://arxiv.org/abs/2402.00782}.

\bibitem[Chaudhari et~al.(2024)Chaudhari, Aggarwal, Murahari, Rajpurohit, Kalyan, Narasimhan, Deshpande, and da~Silva]{chaudhari2024rlhfdecipheredcriticalanalysis}
Shreyas Chaudhari, Pranjal Aggarwal, Vishvak Murahari, Tanmay Rajpurohit, Ashwin Kalyan, Karthik Narasimhan, Ameet Deshpande, and Bruno~Castro da~Silva.
\newblock Rlhf deciphered: A critical analysis of reinforcement learning from human feedback for llms, 2024.
\newblock URL \url{https://arxiv.org/abs/2404.08555}.

\bibitem[Chen et~al.(2024{\natexlab{a}})Chen, Chen, Goldstein, Huang, and Zhou]{cheboh}
Lichang Chen, Jiuhai Chen, Tom Goldstein, Heng Huang, and Tianyi Zhou.
\newblock {I}nstruct{Z}ero: Efficient instruction optimization for black-box large language models.
\newblock In Ruslan Salakhutdinov, Zico Kolter, Katherine Heller, Adrian Weller, Nuria Oliver, Jonathan Scarlett, and Felix Berkenkamp (eds.), \emph{Proceedings of the 41st International Conference on Machine Learning}, volume 235 of \emph{Proceedings of Machine Learning Research}, pp.\  6503--6518. PMLR, 21--27 Jul 2024{\natexlab{a}}.
\newblock URL \url{https://proceedings.mlr.press/v235/chen24e.html}.

\bibitem[Chen et~al.(2024{\natexlab{b}})Chen, Zhou, Zhao, Wan, Zhang, Zhang, and Wen]{chentokenleveld}
Zhipeng Chen, Kun Zhou, Xin Zhao, Junchen Wan, Fuzheng Zhang, Di~Zhang, and Ji-Rong Wen.
\newblock Improving large language models via fine-grained reinforcement learning with minimum editing constraint.
\newblock In Lun-Wei Ku, Andre Martins, and Vivek Srikumar (eds.), \emph{Findings of the Association for Computational Linguistics: ACL 2024}, pp.\  5694--5711, Bangkok, Thailand, August 2024{\natexlab{b}}. Association for Computational Linguistics.
\newblock \doi{10.18653/v1/2024.findings-acl.338}.
\newblock URL \url{https://aclanthology.org/2024.findings-acl.338/}.

\bibitem[Cui et~al.(2024)Cui, Yuan, Ding, Yao, He, Zhu, Ni, Xie, Xie, Lin, Liu, and Sun]{cui2024ultrafeedbackboostinglanguagemodels}
Ganqu Cui, Lifan Yuan, Ning Ding, Guanming Yao, Bingxiang He, Wei Zhu, Yuan Ni, Guotong Xie, Ruobing Xie, Yankai Lin, Zhiyuan Liu, and Maosong Sun.
\newblock Ultrafeedback: Boosting language models with scaled ai feedback, 2024.
\newblock URL \url{https://arxiv.org/abs/2310.01377}.

\bibitem[Daulton et~al.(2022)Daulton, Cakmak, Balandat, Osborne, Zhou, and Bakshy]{daulton2022robustmultiobjectivebayesianoptimization}
Samuel Daulton, Sait Cakmak, Maximilian Balandat, Michael~A. Osborne, Enlu Zhou, and Eytan Bakshy.
\newblock Robust multi-objective bayesian optimization under input noise, 2022.
\newblock URL \url{https://arxiv.org/abs/2202.07549}.

\bibitem[De~Langis \& Kang(2022)De~Langis and Kang]{de2022comparative}
Karin De~Langis and Dongyeop Kang.
\newblock A comparative study on textual saliency of styles from eye tracking, annotations, and language models.
\newblock \emph{arXiv preprint arXiv:2212.09873}, 2022.

\bibitem[Dong et~al.(2023)Dong, Xiong, Goyal, Zhang, Chow, Pan, Diao, Zhang, Shum, and Zhang]{dong2023raftrewardrankedfinetuning}
Hanze Dong, Wei Xiong, Deepanshu Goyal, Yihan Zhang, Winnie Chow, Rui Pan, Shizhe Diao, Jipeng Zhang, Kashun Shum, and Tong Zhang.
\newblock Raft: Reward ranked finetuning for generative foundation model alignment, 2023.
\newblock URL \url{https://arxiv.org/abs/2304.06767}.

\bibitem[Dubois et~al.(2025)Dubois, Galambosi, Liang, and Hashimoto]{dubois2025lengthcontrolledalpacaevalsimpleway}
Yann Dubois, Balázs Galambosi, Percy Liang, and Tatsunori~B. Hashimoto.
\newblock Length-controlled alpacaeval: A simple way to debias automatic evaluators, 2025.
\newblock URL \url{https://arxiv.org/abs/2404.04475}.

\bibitem[Engstrom et~al.(2020)Engstrom, Ilyas, Santurkar, Tsipras, Janoos, Rudolph, and Madry]{engstrom2020implementationmattersdeeppolicy}
Logan Engstrom, Andrew Ilyas, Shibani Santurkar, Dimitris Tsipras, Firdaus Janoos, Larry Rudolph, and Aleksander Madry.
\newblock Implementation matters in deep policy gradients: A case study on ppo and trpo, 2020.
\newblock URL \url{https://arxiv.org/abs/2005.12729}.

\bibitem[Fröhlich et~al.(2020)Fröhlich, Klenske, Vinogradska, Daniel, and Zeilinger]{fröhlich2020noisyinputentropysearchefficient}
Lukas~P. Fröhlich, Edgar~D. Klenske, Julia Vinogradska, Christian Daniel, and Melanie~N. Zeilinger.
\newblock Noisy-input entropy search for efficient robust bayesian optimization, 2020.
\newblock URL \url{https://arxiv.org/abs/2002.02820}.

\bibitem[Fu et~al.(2025)Fu, Zhao, Yao, Wang, Han, and Xiao]{fu2025rewardshapingmitigatereward}
Jiayi Fu, Xuandong Zhao, Chengyuan Yao, Heng Wang, Qi~Han, and Yanghua Xiao.
\newblock Reward shaping to mitigate reward hacking in rlhf, 2025.
\newblock URL \url{https://arxiv.org/abs/2502.18770}.

\bibitem[Gao et~al.(2022)Gao, Schulman, and Hilton]{gao2022scalinglawsrewardmodel}
Leo Gao, John Schulman, and Jacob Hilton.
\newblock Scaling laws for reward model overoptimization, 2022.
\newblock URL \url{https://arxiv.org/abs/2210.10760}.

\bibitem[Gao et~al.(2024)Gao, Xu, Wang, and Cohan]{gaobod}
Yicheng Gao, Gonghan Xu, Zhe Wang, and Arman Cohan.
\newblock {B}ayesian calibration of win rate estimation with {LLM} evaluators.
\newblock In Yaser Al-Onaizan, Mohit Bansal, and Yun-Nung Chen (eds.), \emph{Proceedings of the 2024 Conference on Empirical Methods in Natural Language Processing}, pp.\  4757--4769, Miami, Florida, USA, November 2024. Association for Computational Linguistics.
\newblock \doi{10.18653/v1/2024.emnlp-main.273}.
\newblock URL \url{https://aclanthology.org/2024.emnlp-main.273/}.

\bibitem[Guo et~al.(2023)Guo, Zhao, Tang, Zhao, and Wen]{guotokenlevele}
Geyang Guo, Ranchi Zhao, Tianyi Tang, Wayne~Xin Zhao, and Ji-Rong Wen.
\newblock Beyond imitation: Leveraging fine-grained quality signals for alignment.
\newblock \emph{arXiv preprint arXiv:2311.04072}, 2023.

\bibitem[Gupta et~al.(2022)Gupta, Pacchiano, Zhai, Kakade, and Levine]{gupta2022unpackingrewardshapingunderstanding}
Abhishek Gupta, Aldo Pacchiano, Yuexiang Zhai, Sham~M. Kakade, and Sergey Levine.
\newblock Unpacking reward shaping: Understanding the benefits of reward engineering on sample complexity, 2022.
\newblock URL \url{https://arxiv.org/abs/2210.09579}.

\bibitem[Gupta et~al.(2025)Gupta, Tang, Song, Zhu, Hong, Saha, Gupta, Lee, Kim, Zhu, Agrawal, Pillai, and Keerthi]{guptarewardshapingb}
Aman Gupta, Shao Tang, Qingquan Song, Sirou Zhu, Jiwoo Hong, Ankan Saha, Viral Gupta, Noah Lee, Eunki Kim, Siyu Zhu, Parag Agrawal, Natesh Pillai, and S.~Sathiya Keerthi.
\newblock Alphapo -- reward shape matters for llm alignment, 2025.
\newblock URL \url{https://arxiv.org/abs/2501.03884}.

\bibitem[Hayati et~al.(2021)Hayati, Kang, and Ungar]{hayati2021does}
Shirley~Anugrah Hayati, Dongyeop Kang, and Lyle Ungar.
\newblock Does bert learn as humans perceive? understanding linguistic styles through lexica.
\newblock \emph{arXiv preprint arXiv:2109.02738}, 2021.

\bibitem[Hu et~al.(2020)Hu, Wang, Jia, Wang, Chen, Hao, Wu, and Fan]{hurewardshapinga}
Yujing Hu, Weixun Wang, Hangtian Jia, Yixiang Wang, Yingfeng Chen, Jianye Hao, Feng Wu, and Changjie Fan.
\newblock Learning to utilize shaping rewards: a new approach of reward shaping.
\newblock In \emph{Proceedings of the 34th International Conference on Neural Information Processing Systems}, NIPS '20, Red Hook, NY, USA, 2020. Curran Associates Inc.
\newblock ISBN 9781713829546.

\bibitem[Huang et~al.(2024)Huang, Wu, Hu, and Wang]{huangtokenlevelg}
Chengyu Huang, Zeqiu Wu, Yushi Hu, and Wenya Wang.
\newblock Training language models to generate text with citations via fine-grained rewards.
\newblock In Lun-Wei Ku, Andre Martins, and Vivek Srikumar (eds.), \emph{Proceedings of the 62nd Annual Meeting of the Association for Computational Linguistics (Volume 1: Long Papers)}, pp.\  2926--2949, Bangkok, Thailand, August 2024. Association for Computational Linguistics.
\newblock \doi{10.18653/v1/2024.acl-long.161}.
\newblock URL \url{https://aclanthology.org/2024.acl-long.161/}.

\bibitem[Jain \& Wallace(2019)Jain and Wallace]{jain2019attentionexplanation}
Sarthak Jain and Byron~C. Wallace.
\newblock Attention is not explanation, 2019.
\newblock URL \url{https://arxiv.org/abs/1902.10186}.

\bibitem[Kim et~al.(2025)Kim, Ramachandran, Tavazoee, Kim, Rokhlenko, and Kang]{kim2025align}
Zae~Myung Kim, Anand Ramachandran, Farideh Tavazoee, Joo-Kyung Kim, Oleg Rokhlenko, and Dongyeop Kang.
\newblock Align to structure: Aligning large language models with structural information, 2025.

\bibitem[Kristiadi et~al.(2024)Kristiadi, Strieth-Kalthoff, Skreta, Poupart, Aspuru-Guzik, and Pleiss]{kristiadibof}
Agustinus Kristiadi, Felix Strieth-Kalthoff, Marta Skreta, Pascal Poupart, Al\'{a}n Aspuru-Guzik, and Geoff Pleiss.
\newblock A sober look at llms for material discovery: are they actually good for bayesian optimization over molecules?
\newblock In \emph{Proceedings of the 41st International Conference on Machine Learning}, ICML'24. JMLR.org, 2024.

\bibitem[Li et~al.(2020)Li, Zhou, Dvornek, Gu, Ventola, and Duncan]{li2020-shap-uncertainty}
Xiaoxiao Li, Yuan Zhou, Nicha~C. Dvornek, Yufeng Gu, Pamela Ventola, and James~S. Duncan.
\newblock Efficient shapley explanation for features importance estimation under uncertainty.
\newblock In \emph{Medical Image Computing and Computer Assisted Intervention – MICCAI 2020: 23rd International Conference, Lima, Peru, October 4–8, 2020, Proceedings, Part I}, pp.\  792–801, Berlin, Heidelberg, 2020. Springer-Verlag.
\newblock ISBN 978-3-030-59709-2.
\newblock \doi{10.1007/978-3-030-59710-8_77}.
\newblock URL \url{https://doi.org/10.1007/978-3-030-59710-8_77}.

\bibitem[Lightman et~al.(2023)Lightman, Kosaraju, Burda, Edwards, Baker, Lee, Leike, Schulman, Sutskever, and Cobbe]{lightman2023letsverifystepstep}
Hunter Lightman, Vineet Kosaraju, Yura Burda, Harri Edwards, Bowen Baker, Teddy Lee, Jan Leike, John Schulman, Ilya Sutskever, and Karl Cobbe.
\newblock Let's verify step by step, 2023.
\newblock URL \url{https://arxiv.org/abs/2305.20050}.

\bibitem[Liu et~al.(2024)Liu, Wang, Wang, Chen, Li, Tu, Chu, Li, and Sui]{liubob}
Deyuan Liu, Zecheng Wang, Bingning Wang, Weipeng Chen, Chunshan Li, Zhiying Tu, Dianhui Chu, Bo~Li, and Dianbo Sui.
\newblock Checkpoint merging via bayesian optimization in llm pretraining.
\newblock \emph{CoRR}, 2024.

\bibitem[Loeppky et~al.(2012)Loeppky, Sacks, and Welch]{loeppky2012choosing}
Jason~L Loeppky, Jerome Sacks, and William~J Welch.
\newblock Choosing the sample size of a computer experiment: A practical guide.
\newblock \emph{Technometrics}, 54\penalty0 (4):\penalty0 435--446, 2012.

\bibitem[Lundberg \& Lee(2017)Lundberg and Lee]{lundberg2017shapley}
Scott~M Lundberg and Su-In Lee.
\newblock A unified approach to interpreting model predictions.
\newblock In I.~Guyon, U.~Von Luxburg, S.~Bengio, H.~Wallach, R.~Fergus, S.~Vishwanathan, and R.~Garnett (eds.), \emph{Advances in Neural Information Processing Systems}, volume~30. Curran Associates, Inc., 2017.
\newblock URL \url{https://proceedings.neurips.cc/paper_files/paper/2017/file/8a20a8621978632d76c43dfd28b67767-Paper.pdf}.

\bibitem[López \& Saboyá(2009)López and Saboyá]{lopez2009owenvalues}
S.~López and Martha Saboyá.
\newblock On the relationship between shapley and owen values.
\newblock \emph{Central European Journal of Operations Research}, 17:\penalty0 415--423, 12 2009.
\newblock \doi{10.1007/s10100-009-0100-8}.

\bibitem[Ng et~al.(1999)Ng, Harada, and Russell]{ng1999-policy-invariance}
Andrew~Y. Ng, Daishi Harada, and Stuart~J. Russell.
\newblock Policy invariance under reward transformations: Theory and application to reward shaping.
\newblock In \emph{Proceedings of the Sixteenth International Conference on Machine Learning}, ICML '99, pp.\  278–287, San Francisco, CA, USA, 1999. Morgan Kaufmann Publishers Inc.
\newblock ISBN 1558606122.

\bibitem[Opsahl-Ong et~al.(2024)Opsahl-Ong, Ryan, Purtell, Broman, Potts, Zaharia, and Khattab]{opsahlbog}
Krista Opsahl-Ong, Michael~J Ryan, Josh Purtell, David Broman, Christopher Potts, Matei Zaharia, and Omar Khattab.
\newblock Optimizing instructions and demonstrations for multi-stage language model programs.
\newblock In Yaser Al-Onaizan, Mohit Bansal, and Yun-Nung Chen (eds.), \emph{Proceedings of the 2024 Conference on Empirical Methods in Natural Language Processing}, pp.\  9340--9366, Miami, Florida, USA, November 2024. Association for Computational Linguistics.
\newblock \doi{10.18653/v1/2024.emnlp-main.525}.
\newblock URL \url{https://aclanthology.org/2024.emnlp-main.525/}.

\bibitem[Ouyang et~al.(2024)Ouyang, Wang, Yang, Zhao, Huang, Liu, Pang, Yang, Zhan, Sun, et~al.]{ouyangtokenlevelf}
Yichen Ouyang, Lu~Wang, Fangkai Yang, Pu~Zhao, Chenghua Huang, Jianfeng Liu, Bochen Pang, Yaming Yang, Yuefeng Zhan, Hao Sun, et~al.
\newblock Token-level proximal policy optimization for query generation.
\newblock \emph{arXiv preprint arXiv:2411.00722}, 2024.

\bibitem[Rafailov et~al.(2024)Rafailov, Hejna, Park, and Finn]{rafailov2024rqlanguagemodel}
Rafael Rafailov, Joey Hejna, Ryan Park, and Chelsea Finn.
\newblock From $r$ to $q^*$: Your language model is secretly a q-function, 2024.
\newblock URL \url{https://arxiv.org/abs/2404.12358}.

\bibitem[Razin et~al.(2024)Razin, Zhou, Saremi, Thilak, Bradley, Nakkiran, Susskind, and Littwin]{razin2024vanishinggradientsreinforcementfinetuning}
Noam Razin, Hattie Zhou, Omid Saremi, Vimal Thilak, Arwen Bradley, Preetum Nakkiran, Joshua Susskind, and Etai Littwin.
\newblock Vanishing gradients in reinforcement finetuning of language models, 2024.
\newblock URL \url{https://arxiv.org/abs/2310.20703}.

\bibitem[Ribeiro et~al.(2016)Ribeiro, Singh, and Guestrin]{ribeiro2016whyitrustyou}
Marco~Tulio Ribeiro, Sameer Singh, and Carlos Guestrin.
\newblock "why should i trust you?": Explaining the predictions of any classifier, 2016.
\newblock URL \url{https://arxiv.org/abs/1602.04938}.

\bibitem[Schulman et~al.(2017)Schulman, Wolski, Dhariwal, Radford, and Klimov]{schulman2017proximalpolicyoptimizationalgorithms}
John Schulman, Filip Wolski, Prafulla Dhariwal, Alec Radford, and Oleg Klimov.
\newblock Proximal policy optimization algorithms, 2017.
\newblock URL \url{https://arxiv.org/abs/1707.06347}.

\bibitem[Shen et~al.(2024)Shen, Yang, and Chen]{shen2024principledpenaltybasedmethodsbilevel}
Han Shen, Zhuoran Yang, and Tianyi Chen.
\newblock Principled penalty-based methods for bilevel reinforcement learning and rlhf, 2024.
\newblock URL \url{https://arxiv.org/abs/2402.06886}.

\bibitem[Snoek et~al.(2012)Snoek, Larochelle, and Adams]{snoek2012practicalbayesianoptimizationmachine}
Jasper Snoek, Hugo Larochelle, and Ryan~P. Adams.
\newblock Practical bayesian optimization of machine learning algorithms, 2012.
\newblock URL \url{https://arxiv.org/abs/1206.2944}.

\bibitem[Sutton \& Barto(2018)Sutton and Barto]{Sutton1998}
Richard~S. Sutton and Andrew~G. Barto.
\newblock \emph{Reinforcement Learning: An Introduction}.
\newblock The MIT Press, second edition, 2018.
\newblock URL \url{http://incompleteideas.net/book/the-book-2nd.html}.

\bibitem[Team(2024)]{grattafiori2024llama3herdmodels}
Meta~Llama Team.
\newblock The llama 3 herd of models, 2024.
\newblock URL \url{https://arxiv.org/abs/2407.21783}.

\bibitem[Uesato et~al.(2022)Uesato, Kushman, Kumar, Song, Siegel, Wang, Creswell, Irving, and Higgins]{uesato2022solvingmathwordproblems}
Jonathan Uesato, Nate Kushman, Ramana Kumar, Francis Song, Noah Siegel, Lisa Wang, Antonia Creswell, Geoffrey Irving, and Irina Higgins.
\newblock Solving math word problems with process- and outcome-based feedback, 2022.
\newblock URL \url{https://arxiv.org/abs/2211.14275}.

\bibitem[von Werra et~al.(2020)von Werra, Belkada, Tunstall, Beeching, Thrush, Lambert, Huang, Rasul, and Gallouédec]{vonwerra2022trl}
Leandro von Werra, Younes Belkada, Lewis Tunstall, Edward Beeching, Tristan Thrush, Nathan Lambert, Shengyi Huang, Kashif Rasul, and Quentin Gallouédec.
\newblock Trl: Transformer reinforcement learning.
\newblock \url{https://github.com/huggingface/trl}, 2020.

\bibitem[Wu et~al.(2023)Wu, Hu, Shi, Dziri, Suhr, Ammanabrolu, Smith, Ostendorf, and Hajishirzi]{wu2023finegrainedhumanfeedbackgives}
Zeqiu Wu, Yushi Hu, Weijia Shi, Nouha Dziri, Alane Suhr, Prithviraj Ammanabrolu, Noah~A. Smith, Mari Ostendorf, and Hannaneh Hajishirzi.
\newblock Fine-grained human feedback gives better rewards for language model training, 2023.
\newblock URL \url{https://arxiv.org/abs/2306.01693}.

\bibitem[Xie et~al.(2024)Xie, Zhao, Wu, Liu, Luo, Zhong, Yang, and Yu]{xie2024text2rewardrewardshapinglanguage}
Tianbao Xie, Siheng Zhao, Chen~Henry Wu, Yitao Liu, Qian Luo, Victor Zhong, Yanchao Yang, and Tao Yu.
\newblock Text2reward: Reward shaping with language models for reinforcement learning, 2024.
\newblock URL \url{https://arxiv.org/abs/2309.11489}.

\bibitem[Yang et~al.()Yang, Robeyns, Coste, Shi, Wang, Ammar, and Aitchison]{yangboc}
Adam~X Yang, Maxime Robeyns, Thomas Coste, Zhengyan Shi, Jun Wang, Haitham~Bou Ammar, and Laurence Aitchison.
\newblock Bayesian reward models for llm alignment.
\newblock In \emph{ICML 2024 Workshop on Structured Probabilistic Inference $\{$$\backslash$\&$\}$ Generative Modeling}.

\bibitem[Yoon et~al.(2024)Yoon, Yoon, Eom, Han, Nam, Jo, On, Hasegawa-Johnson, Kim, and Yoo]{yoontokenlevelc}
Eunseop Yoon, Hee~Suk Yoon, SooHwan Eom, Gunsoo Han, Daniel Nam, Daejin Jo, Kyoung-Woon On, Mark Hasegawa-Johnson, Sungwoong Kim, and Chang Yoo.
\newblock Tlcr: Token-level continuous reward for fine-grained reinforcement learning from human feedback.
\newblock In \emph{Findings of the Association for Computational Linguistics ACL 2024}, pp.\  14969--14981, 2024.

\bibitem[Zeng et~al.(2024)Zeng, Liu, Ma, Yang, Zhang, and Wang]{zengtokenlevelb}
Yongcheng Zeng, Guoqing Liu, Weiyu Ma, Ning Yang, Haifeng Zhang, and Jun Wang.
\newblock Token-level direct preference optimization.
\newblock In \emph{Proceedings of the 41st International Conference on Machine Learning}, pp.\  58348--58365, 2024.

\bibitem[Zhang et~al.(2023)Zhang, Khanduri, Tsaknakis, Yao, Hong, and Liu]{zhang2023introductionbileveloptimizationfoundations}
Yihua Zhang, Prashant Khanduri, Ioannis Tsaknakis, Yuguang Yao, Mingyi Hong, and Sijia Liu.
\newblock An introduction to bi-level optimization: Foundations and applications in signal processing and machine learning, 2023.
\newblock URL \url{https://arxiv.org/abs/2308.00788}.

\bibitem[Zheng et~al.(2023{\natexlab{a}})Zheng, Chiang, Sheng, Zhuang, Wu, Zhuang, Lin, Li, Li, Xing, Zhang, Gonzalez, and Stoica]{zheng2023judgingllmasajudgemtbenchchatbot}
Lianmin Zheng, Wei-Lin Chiang, Ying Sheng, Siyuan Zhuang, Zhanghao Wu, Yonghao Zhuang, Zi~Lin, Zhuohan Li, Dacheng Li, Eric~P. Xing, Hao Zhang, Joseph~E. Gonzalez, and Ion Stoica.
\newblock Judging llm-as-a-judge with mt-bench and chatbot arena, 2023{\natexlab{a}}.
\newblock URL \url{https://arxiv.org/abs/2306.05685}.

\bibitem[Zheng et~al.(2023{\natexlab{b}})Zheng, Dou, Gao, Hua, Shen, Wang, Liu, Jin, Liu, Zhou, Xiong, Chen, Xi, Xu, Lai, Zhu, Chang, Yin, Weng, Cheng, Huang, Sun, Yan, Gui, Zhang, Qiu, and Huang]{zheng2023secretsrlhflargelanguage}
Rui Zheng, Shihan Dou, Songyang Gao, Yuan Hua, Wei Shen, Binghai Wang, Yan Liu, Senjie Jin, Qin Liu, Yuhao Zhou, Limao Xiong, Lu~Chen, Zhiheng Xi, Nuo Xu, Wenbin Lai, Minghao Zhu, Cheng Chang, Zhangyue Yin, Rongxiang Weng, Wensen Cheng, Haoran Huang, Tianxiang Sun, Hang Yan, Tao Gui, Qi~Zhang, Xipeng Qiu, and Xuanjing Huang.
\newblock Secrets of rlhf in large language models part i: Ppo, 2023{\natexlab{b}}.
\newblock URL \url{https://arxiv.org/abs/2307.04964}.

\bibitem[Zhong et~al.(2025)Zhong, Shan, Feng, Xiong, Cheng, Zhao, He, Bian, and Wang]{zhong2025dpomeetspporeinforced}
Han Zhong, Zikang Shan, Guhao Feng, Wei Xiong, Xinle Cheng, Li~Zhao, Di~He, Jiang Bian, and Liwei Wang.
\newblock Dpo meets ppo: Reinforced token optimization for rlhf, 2025.
\newblock URL \url{https://arxiv.org/abs/2404.18922}.

\end{thebibliography}
\bibliographystyle{colm2025_conference}

\appendix

\section{Computing the Explainability Scores from a LLM Reward Model}\label{appendix:a}

We discuss a simple setting of computing the SHAPley scores and also well-define its terms such as the simplified mapping $z'$, the mapping function $h_x$, and show by example the satisfaction of the faithfulness property: $g(z') \approx f(h_x(z'))$.

Consider an input text $x$: "I like apples" that receives a reward score $r(x) = 2.1$. Then we consider 3 tokens 

\begin{itemize}
    \item Token 1: I
    \item Token 2: like
    \item Token 3: apples
\end{itemize}

The binary vector $z' = [z_1', z_2', z_3']$ has each $z_i' \in{0,1}$ which masks each token as (1) present or (0) absent. The mapping function $h_x$ reconstructs valid text from the simplified input. For example, if $z' = [1, 0 ,1]$ then $h_x(z')$ might return "\texttt{I [MASK] apples}" and $r(h_x(z'))$ is the true reward for the masked text. We draw a table for each permutation and a toy example of the reward for each masked text

\begin{table}[h]
    \centering
    \begin{tabular}{ccc}
        \toprule
        $z'$ & Masked Text $h_x(z')$ & $r(h_x(z'))$ \\
        \midrule
        \texttt{[0,0,0]} & \texttt{[MASK] [MASK] [MASK]} & 0 \\
        \texttt{[1,0,0]} & \texttt{I [MASK] [MASK]} & 0.3 \\
        \texttt{[0,1,0]} & \texttt{[MASK] like [MASK]} & 0.5 \\
        \texttt{[0,0,1]} & \texttt{[MASK] [MASK] apples} & 1.2 \\
        \texttt{[1,1,0]} & \texttt{I like [MASK]} & 0.9 \\
        \texttt{[1,0,1]} & \texttt{I [MASK] apples} & 1.3 \\
        \texttt{[0,1,1]} & \texttt{[MASK] like apples} & 1.7 \\
        \texttt{[1,1,1]} & \texttt{I like apples} & 2.1 \\
        \bottomrule
    \end{tabular}
\end{table}

From this mapping, we can estimate token-level scores at intermediate states by computing the explainability metric or the feature attribution for each token to interpret each prediction of the reward model. Here, we consider the SHAPley kernel, which provides a natural way to identify the \textit{marginal} contribution of each token $i$ concerning the overall reward:
\begin{equation}\label{appendix:shap-kernel}
    \phi_i(f, x) = \sum_{s \subseteq x \backslash \{ i \}} \frac{|s|! (|x| - |s| - 1)!}{|x|!} [f(s \cup \{i\}) - f_x(s)]
\end{equation}
where $s$ is all subsets of tokens excluding token $i$ that also respect the order of the original sentence, and $x$ is the bag of all tokens in the sequence.

Here, we have $x = \{1,2,3\}$ as the bag of all tokens from the sequence, and $s$ constructs each subset of tokens but also maintains the order of the original sentence. Here, we can define the terms to properly solve Eq. 2 as a penalized linear regression:

\begin{align*}
    \Omega(g) &= 0 \\
    \pi_{x}(z') &= \frac{|z'|!(|x| - |z'| - 1)!}{|x|!} = \frac{|s|! (|x| - |s| - 1)!}{|x|!} \\
    \mathcal{L}(f, g, \pi_x) &= \sum_{z'\in Z} \left [ f(h_x(z')) -g(z') \right]^2 \pi_x(z')
\end{align*}

Now, in relation to the simplified input $z'$, we consider a subset $s$ with the token $\{2\}$ which corresponds to $z' = [0, 1, 0]$. Additionally $s \cup \{2\} = \{1, 2\}$ corresponds to $z' = [1, 1, 0].$ Hence, using Eq. 3 we can compute the feature attribution for each token $\phi_i$. We continue with an example computation of the feature attribution for the token "i".

\begin{example}[] We compute the marginal contribution of element \( i = 1 \) across all subsets \( s \subseteq x \setminus \{1\} \), using the Shapley value formula.

\textbf{Case 1:} \( s = \varnothing \)
\begin{align*}
\text{Weight:} \quad &\frac{|s|!(|x| - |s| - 1)!}{|x|!} 
= \frac{0! \cdot (3 - 0 - 1)!}{3!} 
= \frac{1 \cdot 2!}{6} 
= \frac{2}{6} 
= \frac{1}{3} \\
\text{Marginal:} \quad &f(s \cup \{1\}) - f(s) 
= f(\{1\}) - f(\varnothing) 
= 0.3 - 0 = 0.3 \\
\text{Contribution:} \quad &\frac{1}{3} \cdot 0.3 = 0.1
\end{align*}

\textbf{Case 2:} \( s = \{2\} \)
\begin{align*}
\text{Weight:} \quad &\frac{1! \cdot (3 - 1 - 1)!}{3!} 
= \frac{1 \cdot 1!}{6} 
= \frac{1}{6} \\
\text{Marginal:} \quad &f(\{1, 2\}) - f(\{2\}) 
= 0.9 - 0.5 = 0.4 \\
\text{Contribution:} \quad &\frac{1}{6} \cdot 0.4 
= 0.066\overline{6} \approx 0.067
\end{align*}

\textbf{Case 3:} \( s = \{3\} \)
\begin{align*}
\text{Weight:} \quad &\frac{1! \cdot (3 - 1 - 1)!}{3!} 
= \frac{1 \cdot 1!}{6} 
= \frac{1}{6} \\
\text{Marginal:} \quad &f(\{1, 3\}) - f(\{3\}) 
= 0.6 - 0.2 = 0.4 \\
\text{Contribution:} \quad &\frac{1}{6} \cdot 0.4 
= 0.067
\end{align*}

\textbf{Case 4:} \( s = \{2, 3\} \)
\begin{align*}
\text{Weight:} \quad &\frac{2! \cdot (3 - 2 - 1)!}{3!} 
= \frac{2 \cdot 0!}{6} 
= \frac{2}{6} 
= \frac{1}{3} \\
\text{Marginal:} \quad &f(\{1, 2, 3\}) - f(\{2, 3\}) 
= 1.0 - 0.7 = 0.3 \\
\text{Contribution:} \quad &\frac{1}{3} \cdot 0.3 
= 0.1
\end{align*}
\end{example}

Continuing this sequence gives eventually the rest of the terms in which summing them computes $\phi_1$ as:
\begin{align*}
    \phi_1 = 0.1 + 0.067 + 0.017 + 0.133 \approx 0.32 
\end{align*}

Continuing this with the rest of the SHAPley values gives eventually 
\begin{align*}
    \phi_1 = 0.32, \phi_2 = 0.62, \phi_3 = 1.17
\end{align*}
Hence, we see that the summation of all feature attributions give approximately the full reward as well as satisfying the additive surrogate 

\begin{equation*}
    g(z') = r(\varnothing) + \sum_{i=1}^3 \phi_i z' \approx r(h_x(z'))
\end{equation*}
 
In this example, it is important to note that we overview the standard SHAPley formula, which requires computing every partition of tokens, resulting in $2^k$ operations. However, in practice, we employ the Owen values, which provably produce the same solution as SHAPley values \citep{lopez2009owenvalues} and also reduce the necessary operations drastically, requiring to compute only $x^2$ (where x is the total sequence length) operations per prediction.


\section{Proof of Policy Invariance of Explainability Methods}\label{appendix:policy-invariance}

\begin{prop}
    Given a policy $\pi_\theta$ and a Markov Decision Process $\mathcal{M} = (\mathcal{S}, \mathcal{A}, \mathcal{P}, \gamma, R)$, any reward shaping function $R'$ in the family of additive feature attribution methods follows a potential-based shaping function and then the optimal $\pi_\theta$ for $R'$ is also optimal for the original $R$.
\end{prop}

\begin{proof} 
    We first have $R'(s,a,s') = R(s,a ,s') + F(s,a,s')$ where $F$ is a potential-based shaping function defined as $F(s,a,s') = \gamma \Phi(s') -\Phi(s) $ for $\Phi : S \mapsto \mathbb{R}$. We also have by locality from Eq. 1 that $g(x') \approx R_\phi(s,a,s')$ or:
    \begin{equation}
        \sum_{i=1}^{s'} \phi_i(f, x) \approx R(s, a, s') 
    \end{equation}
    where $s'$ refers to the index up to state $s'$ in the sequence. Then, without loss of generality, consider the case with heuristics $\boldsymbol{\mathcal{E}} = \{ \phi_{s'}, 1\}$ where the constant is just the identity for the original sparse reward. We omit the case for the constant, as it is known that any linear transformation also preserves the optimal policy. Then, by substituting Eq. \ref{eq:shaping} for $F$ with $\gamma = 1$ since we care about all future states equally, we have:
    \begin{align}
        \Phi(s') - \Phi(s) = w_2 \phi_{s'}(R,x)  
    \end{align}
    Then defining $\Phi(s) = w_2 \sum\limits_{i=1}^{s} \phi_i(R, x)$ we have:
    \begin{equation}
        \begin{aligned}
            \Phi(s') - \Phi(s) &=  w_2 \sum_{i=1}^{s'} \phi_i(R, x) -  w_2 \sum_{i=1}^{s} \phi_i(R, x) =  w_2 \phi_{s'}(R,x) 
        \end{aligned}
    \end{equation}
    Since by Eq. 8 we have that 
    \begin{equation*}
        \sum_{i=1}^{s'} \phi_i(R,x) -\sum_{i=1}^s\phi_i(R,x) = \phi_{s'}(R,x)
    \end{equation*}
\end{proof}
\section{Computation Overhead of Explanation Methods}\label{appendix:wall-time}

To measure the computational overhead of explanation methods, we can represent the time complexity required for inference of the reward model as $\mathcal{O}(N \times F(k))$ where $F(k)$ represents the number of floating point operations required in one forward pass on sequence of length $k$, then our method is upper-bounded by $\mathcal{O}(N \times k^2 \times F(k))$ computations when also calculating SHAPley values. We run a simple test measuring the total wall clock time, varying batch size and generation length to run inference with and without explainability methods:

\begin{figure}[h]
    \centering
    \includegraphics[width=0.5\linewidth]{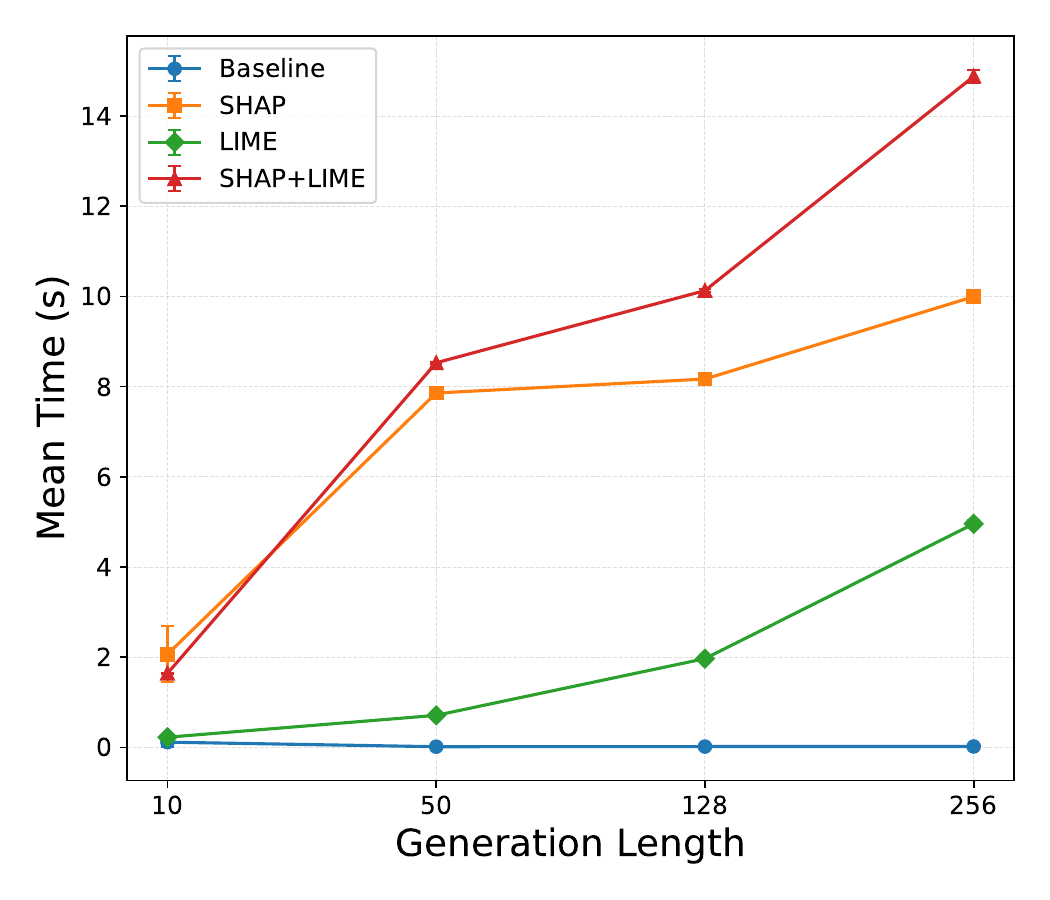}
    \label{fig:wall-clock}
\end{figure}

\newpage
\section{Practical Implementation Details}\label{appendix:impl-algorithm}

\begin{algorithm}[H] 
\caption{Explainable Reward Shaping with Bayesian Optimization}
\label{alg:bo-ppo}
\begin{algorithmic}[1]
    \State \textbf{Initialize} a Gaussian Process (GP) surrogate model: $BO \gets \text{GP}$
    \For{$k = 1,\dots,n$} \Comment{Bayesian Optimization iterations}
        \State $w_k \gets BO.\text{next\_action()}$
        \For{$x_i \in \mathcal{D}_{1:n}$} \Comment{Subsampled set of points $\mathcal{D}$ and train PPO}
            \State $y_i \gets \pi_\theta(x_i)$
            \State $r(x_i,y_i);\;\boldsymbol{\mathcal{E}} \gets \mathcal{E}(r)$ \Comment{Compute sparse
            and token-level scores (Eq. ~2)}
            \State $r' \gets w_k^\top \boldsymbol{\mathcal{E}}$
            \State Update $\theta$ via $\nabla_\theta J(\pi_\theta; r')$
        \EndFor
        \State $r^{val} \gets \frac{1}{m}\sum_{i=1}^{m} r\bigl(x_i,\pi_\theta(x_i)\bigr)$ \Comment{Compute avg. validation reward}
        \State $BO.\text{update}(w_k, r^{val})$
        \If{$r^{val} > r^{best}$}\State Update the PPO checkpoint for next BO iteration
        \EndIf
    \EndFor
    \State $w_{\text{best}} \gets BO.\text{best\_params()}$ \Comment{Retrieve the best parameters}
    \State \textbf{Train new PPO} using $r' \gets w_{\text{best}}^\top \boldsymbol{\mathcal{E}}$
\end{algorithmic}
\end{algorithm}

\begin{table}[h]
\centering
\begin{tabular}{|l|l|}
\hline
\textbf{Component} & \textbf{Link} \\
\hline
Dataset & \href{https://huggingface.co/datasets/Anthropic/hh-rlhf}{Anthropic/hh-rlhf} \\
\hline
Reward Model Base & \href{https://huggingface.co/weqweasdas/hh_rlhf_rm_open_llama_3b}{weqweasdas/hh\_rlhf\_rm\_open\_llama\_3b} \\
\hline
SFT Model Base & \href{https://huggingface.co/VMware/open-llama-7b-open-instruct}{VMware/open-llama-7b-open-instruct} \\
\hline
\end{tabular}
\end{table}

\begin{table}[h]
\centering
\begin{tabular}{|l|l|}
\hline
\textbf{Component} & \textbf{Link} \\
\hline
Dataset & \href{https://huggingface.co/datasets/princeton-nlp/llama3-ultrafeedback}{princeton-nlp/llama3-ultrafeedback} \\
\hline
Reward Model Base & \href{https://huggingface.co/meta-llama/Llama-3.2-1B-Instruct}{meta-llama/Llama-3.2-1B-Instruct} \\
\hline
SFT Model Base & \href{https://huggingface.co/meta-llama/Llama-3.2-3B-Instruct}{meta-llama/Llama-3.2-3B-Instruct} \\
\hline
\end{tabular}
\end{table}

\clearpage
\section{Qualitative Examples}
\begin{table}[ht]
\renewcommand{\arraystretch}{1.4}
\centering
\begin{tabular}{|p{4.0cm}|p{8.0cm}|p{1.5cm}|}
\hline
\multicolumn{3}{|p{13.5cm}|}{\textbf{Prompt:} Can you describe the ways in which insomnia is treated?} \\
\hline
\textbf{Model} & \textbf{Generated Text} from HH-RLHF Test split & \textbf{Score} \\
\hline
\texttt{SHAP} (no BO) & 
The main treatments for insomnia are cognitive behavioral therapy (CBT), pharmaceutical drugs, and sleep hygiene practices. Cognitive behavioral therapy involves techniques such as sleep hygiene education, relaxation training, and cognitive restructuring, and is shown to be significantly more effective than medication or simple advice for short-term insomnia. Medication is generally used as a second-line treatment for chronic insomnia, and can include hypnotics such as benzodiazepines or non-benzodiazepine sedatives and inhibitors of the wake-promoting hormones such as histamine and melatonin. Sleep hygiene practices such as avoiding
 & \textbf{4.63} \\
\hline
\texttt{SHAP} $(d=2)$ & 
"Insomnia is typically treated with one of the three following methods:

\begin{enumerate}
    \item Medication
    \item Cognitive behavioral therapy (CBT)
    \item Sleep hygiene practices, such as regular sleep schedules and avoiding heavy meals before bedtime.
\end{enumerate}
Medications like benzodiazepines and selective serotonin reuptake inhibitors (SSRIs) can help with sleep problems. 
Cognitive behavioral therapy can be helpful in changing maladaptive thoughts and behaviors around sleep, such as staying up too late, which prolongs the time spent awake. It can also help teach sleep-promoting" & \textbf{6.5} \\
\hline
\texttt{SHAP + ATTN} $(d=3)$ & 
Insomnia is treated based on the underlying cause. The most common treatments include medications such as Benzodiazepines (e.g. triazolam, temazepam, alprazolam), Z-drugs (e.g. doxepin, zolpidem), or Minor tranquillizers (e.g. buspirone). Cognitive behavioral therapy (CBT) and relaxation techniques are also common treatments.
\texttt{</s>} & \textbf{8.63} \\
\hline
\texttt{SHAP + LIME} $(d=3)$ & 
Treatment for insomnia includes changes in lifestyle and environmental factors, as well as medications and other therapies. Common treatments include cognitive-behavioural therapy, which focuses on identifying and changing thought patterns and other behavioural factors that may contribute to insomnia, and medication that may help regulate the sleep cycle. Other treatments may include relaxation techniques, such as meditation or breathing exercises, which can help calm the mind and promote sleep.
 \texttt{</s>} & \textbf{9.31} \\
\hline
\end{tabular}
\caption{Generation topography across different models as input dimension increases for BO.}
\end{table}

\end{document}